\documentclass[11pt,letterpaper]{article}
\usepackage[margin=1in]{geometry}
\usepackage{times}
\usepackage{graphicx}
\usepackage[authoryear,round]{natbib}

\usepackage[utf8]{inputenc} %
\usepackage[T1]{fontenc}    %
\usepackage{url}            %
\usepackage{booktabs}       %
\usepackage{amsfonts}       %
\usepackage{nicefrac}       %
\usepackage{microtype}      %
\usepackage{xcolor}         %
\usepackage{amsmath,amssymb,amsthm,bm,mathtools}
\usepackage{physics}
\usepackage{color}
\usepackage{tcolorbox}
\usepackage{wrapfig}
\usepackage{placeins}
\usepackage{float}
\usepackage{titlesec}
\titleformat{\section}{\normalfont\large\scshape}{\thesection}{0.8em}{}
\titleformat{\subsection}{\normalfont\normalsize\bfseries}{\thesubsection}{0.8em}{}
\titleformat{\subsubsection}{\normalfont\normalsize\itshape}{\thesubsubsection}{0.8em}{}
\usepackage[font=small,labelfont=bf,labelsep=period,skip=5pt]{caption}
 \usepackage{ulem} 
\theoremstyle{plain}
\newtheorem{theorem}{Theorem}[section]

\newtheorem{lemma}[theorem]{Lemma}

\theoremstyle{definition}

\newcommand{\gL}{\nabla \mathcal{L}}

\usepackage{url}
\usepackage{array}
\usepackage{hyperref}
\renewcommand{\L}{\mathcal{L}}
\hypersetup{
  colorlinks   = true, %
  urlcolor     = [rgb]{0.18,0.23,0.28}, %
  linkcolor    = [rgb]{0.18,0.23,0.28}, %
  citecolor   = [rgb]{0.10,0.25,0.55} %
}

\makeatletter
\renewcommand{\maketitle}{%
  \par\begingroup
  \renewcommand{\thefootnote}{\fnsymbol{footnote}}
  \centering
  \vspace*{-38pt}
  {\normalfont\fontsize{22}{26}\selectfont\scshape \@title\par}
  \vspace{18pt}
  {\normalsize\lineskip=5pt
    \begin{tabular}[t]{c}\@author\end{tabular}\par}
  \vspace{14pt}
  \@thanks
  \endgroup
  \setcounter{footnote}{0}
}
\makeatother
\renewenvironment{abstract}{%
  \begin{list}{}{\setlength{\leftmargin}{18pt}\setlength{\rightmargin}{18pt}}%
  \item\relax\small
  {\centering\normalfont\scshape Abstract\par}\vspace{5pt}\noindent\ignorespaces
}{\end{list}\vspace{-6pt}}

\title{Measuring the local shape of\\ learning-rate transfer curves}

\author{%
  Satoki Ishikawa$^{1,2}$\thanks{Email: \texttt{\{ishikawa, rioyokota\}@rio.scrc.iir.isct.ac.jp}; \texttt{karakida.ryo@aist.go.jp}.}, Rio Yokota$^{1}$, Ryo Karakida$^{2,1}$ \\
  $^1$ Institute of Science Tokyo ~~
  $^2$ AIST
}
\date{}

\begin{document}

\maketitle

\begin{abstract}
Understanding the learning-rate (LR) dependence of trained model performance is essential for better optimization of neural networks, but the appropriate size of the LR depends on learning dynamics, architecture, and their hidden biases, and thus our understanding of this complex dependence remains limited.
The behavior of the loss curve as a function of the LR has recently been observed in extensive experiments on LR transfer across different settings, such as model sizes, and it is becoming more important to explore methodologies for quantitatively characterizing the shape of this LR curve.
In this work, we provide two relational equations that give a good empirical explanation of the gradient of the LR transfer curve, i.e., the LR gradient.
One is a theoretical formula that holds at the edge of stability, and the other
is an approximate expression obtained by naively treating this gradient as a hyper-parameter gradient.
Using these two relational equations in a complementary manner, we show a lower bound on the gradient of the curve and a zero gradient under specific conditions for normalization layers in theory.
The approximate expression also applies to a wide range of experimental settings and allows us to decompose the LR gradient into layer-wise components and quantify the strong contribution of specific layers to the curve shape.
Thus, this serves as a foundation for better understanding the internal mechanisms of LR dependence.

\end{abstract}

\section{Introduction}

The performance of deep neural networks depends not only on architectural advances and large-scale data, but also crucially on optimization and the tuning of hyperparameters (HPs)~\citep{smith2018,choi2019empirical,cohen2021gradient}.
In practice,
when we sweep multiple learning rates (LRs) and plot the final loss as a function of the LR, the curve typically has a U-shaped profile and its detailed shape varies considerably with model size, architecture, or other HPs ~\citep{yang2021tuning,everett24,yang2020feature,noci2024super,shah2025practical}.  
As the previous work has demonstrated, if the curve shapes remain robust across different model sizes, one can reuse an LR found on a small model for a much larger one, which reduces the computational cost of HP search and  is known as LR transfer. 
Although such curves are essential for understanding model performance beyond the context of transfer across model sizes \citep{wortsman2024smallscale,lewkowycz2020large,barrett2020implicit}, we refer to them as LR transfer curves, following this literature.

While experimental studies have gradually identified key factors governing the shape of the LR transfer curve, such as parameterization, weight decay, and  weight-feature alignment~\citep{everett24,kosson2025weight,haas2026on}, our theoretical understanding is still limited~\citep{hayou2026a,ghosh2026understanding}. First of all, the performance of trained models depends on various factors  
including learning dynamics, architectural choices and implicit biases~\citep{everett24,noci2024super,smith2021origin}, and this complexity makes the tractable evaluation of the curve highly non-trivial. 
In particular, modern deep learning typically uses large LRs beyond the regime where the descent lemma of classical optimization holds~\citep{barrett2020implicit,lewkowycz2020large}, and Edge of Stability (EoS) phenomena widely appear~\citep{cohen2021gradient,arora2022understanding,damian2023self}.
Such a large LR causes highly nonlinear behavior of the learning dynamics and thus that of the trained model performance.

Another reason for the difficulty of characterizing the LR transfer curve is that, because the curve is a function of the LR, which is a hyperparameter, characterizing its geometry naturally requires the {\it hypergradient}.
The computation of the hypergradient has been classically explored in the context of gradient-based HP optimization~\citep{maclaurin2015gradient,lorraine2020optimizing}. However, 
the hypergradient with respect to the LR has rarely been explored, and more crucially, its naive computation in deep architectures easily leads to gradient explosion~\citep{maclaurin2015gradient} or short-horizon bias~\citep{wu2018understanding}, making it difficult to accurately estimate the LR gradient.

In this work, to address the above-mentioned difficulties, we reveal hidden laws on how the local shape of the LR transfer curve is determined. 
More specifically, we focus on quantitatively analyzing the slope of the LR curve, that is, the first-order derivative which we refer to as the {\it LR gradient}.
The detailed contributions are summarized in the following two points. 
\begin{itemize}
\item We propose two relational expressions that quantitatively explain the LR gradient well in experiments.
First, assuming that EoS holds, we theoretically derive a formula for the LR gradient ({\it EoS-LR gradient}),  through a trick based on the Lagrange multiplier method in Section \ref{sec3}.
Second, to measure the LR gradient  in a broader range of settings, we also consider the original hypergradient update of the LR gradient. Assuming slow dynamics of the hypergradient, we identify a specific fixed-point expression of the LR gradient ({\it FP-LR gradient}) which works well empirically, as shown in Section~\ref{sec4}.
\item 
Both can explain the effect of normalization layers (Section \ref{sec5_norm}) and provide some bounds on the LR gradient (Section \ref{sec5_bounds}), but EoS-LR can explain the experimental results in much finer detail.
In particular, it enables us to characterize a flat transfer curve caused by a specific combination of  weight decay and normalization (Theorem \ref{thm2}).
Meanwhile, FP-LR does not require an explicit EoS constraint for each optimizer and allows decomposition into layer-wise contributions under a block-diagonal Hessian approximation (Section~\ref{sec5_layer}). These profiles identify large LR sensitivities in the final layer of ResNet and the embedding and output layers of GPT (Figure~\ref{fig:gpt-qknorm-lrgrad-profile}). In ResNet, changing the output multiplier changes accuracy and LR transfer across widths under $\mu$P, whereas changing the input multiplier has little effect (Figure~\ref{fig:base_width_resnet_cifar}).

\end{itemize}
Thus, this work provides a first step toward understanding the universal shape of the LR curves for LR tuning and transfer.

\section{Preliminaries}
\subsection{Related Work}

\paragraph{LR transfer curve and sensitivity}
 
\citet{yang2020feature} first analyzed how HPs should be scaled with network widths to preserve feature learning in the infinite-width limit, and introduced the Maximal Update Parametrization ($\mu$P). 
Then, \citet{yang2021tuning} empirically demonstrated that  the optimal LR can remain consistent across model sizes under $\mu$P and this prompted the expansion of LR transfer research~\citep{chizat2024feature,noci2024super,bordelon2025deep,ghosh2026understanding,hayou2026a}. 
While $\mu$P theory describes an appropriate LR scale early in training, recent empirical studies have reported that its predictions do not necessarily hold as training progresses. For instance, recent empirical evidence suggests that stable LR transfer depends on a combination of various factors, including weight-feature alignment~\citep{everett24}, weight decay~\citep{kosson2025weight}, loss function and normalization layers~\citep{haas2026on}.
In contrast, our relational expressions of the LR gradient directly analyze the LR curve and provide a quantitative explanation of these factors in a somewhat unified manner from a new, independent perspective.

Independently of LR transfer, it is known that the LR can lie beyond the classical optimization regime, where the descent lemma guarantees convergence, and that the discrete nature of the updates with large LRs strongly influences the dynamics and trained model performance~\citep{lewkowycz2020large,barrett2020implicit,smith2021origin}.
In particular, EoS is a phenomenon which appears after the early phase of training and is observed across a wide range of modern optimizers and architectures, in which the LR satisfies a specific relation with the loss curvature~\citep{cohen2021gradient, arora2022understanding,kalra2026scalable,andreyev2024edge,islamov2026non}.
While almost no study has examined LR curves under this regime, \citet{noci2024super} empirically confirmed that the largest eigenvalue of the Hessian becomes width-independent under EoS, supporting better LR transfer.
Thus, our work provides the first result that reveals, through some theoretical evaluation, how EoS governs the LR gradient.

\paragraph{Computation of hypergradients}

Hypergradients have been developed in studies of HP optimization~\citep{bengio2000gradient,maclaurin2015gradient,franceschi2017forward}.
However, a key issue is that  standard hypergradients easily suffer from gradient explosion, since their computation essentially involves products of many Jacobian matrices~\citep{maclaurin2015gradient}.
Moreover, a short-horizon bias is also known to occur, leading to underestimation of the HPs to achieve an instantaneous loss decrease~\citep{wu2018understanding}.
Our setting is different because we use hypergradients to evaluate the LR sensitivity of trained models rather than for HP optimization, but iterative hypergradient computation is still unstable, as is observed in Section \ref{sec4}.
In contrast, we find that, although its derivation is heuristic, a certain fixed-point expression provides an accurate approximation to the LR gradient in our experiments.
Note that in HP optimization, fixed-point formulations have been explored mainly for other HPs such as weight decay and data transformations, rather than LR~\citep{lorraine2020optimizing}.

\subsection{The shape of LR transfer curves}

\begin{figure}
    \centering
    \includegraphics[width=\linewidth]{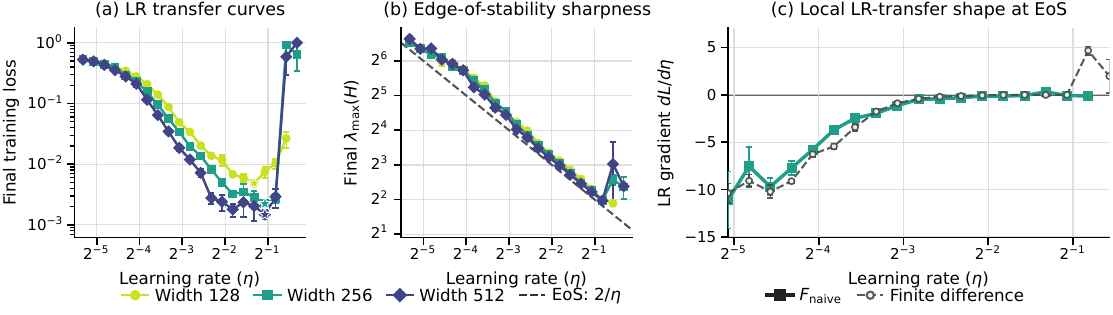}
    \caption{
\textbf{Learning-rate transfer curves and their local shape at the edge of stability.}
(a) LR transfer curves show how the final training loss depends on the learning
rate $\eta$ across network widths.
(b) We focus on the EoS regime,
where the largest Hessian eigenvalue satisfies
$\lambda_{\max}(H) \approx 2/\eta$.
(c) The EoS-LR gradient of $F_{\mathrm{naive}}$ closely tracks the
finite difference estimate, thereby capturing the local shape of the
LR transfer curve (shown for width 256).
}
    \label{fig:fig1}
\end{figure}

Previous work has empirically examined the shape of the LR transfer curve of the loss function $\mathcal{L}(\theta)$ through grid search. This means that the local shape is given by a {\it finite difference}:
\begin{equation}
\Delta \mathcal{L} = \mathcal{L}(\theta_*(\eta + \Delta)) - \mathcal{L}(\theta_*(\eta)),
\label{fin_diff}
\end{equation}
where $\eta$ is the LR, $\theta_*$ denotes the trained model parameters, and $\Delta$ is the resolution of the grid search over the LR.
Figure~\ref{fig:fig1}(a) shows a typical LR transfer curve.  The shape of this curve can depend on multiple factors, such as parameterization, optimizers, learning dynamics, and architectures~\citep{yang2021tuning,kosson2025weight,haas2026on,everett24,noci2024super}.
However, identifying universal behavior empirically is highly challenging in such complex settings, since relying solely on this finite-difference expression requires exhaustive experiments and plots across diverse settings. A more transparent expression for the curve is therefore desirable.

A natural first attempt is to assume that $\Delta$ is sufficiently small and consider the first derivative of the loss; however, this requires computing the hypergradient $\partial \theta_*/\partial \eta$, whose evaluation in modern neural networks is by no means trivial:
\begin{equation}
\frac{d \mathcal{L}}{d \eta} = \nabla_\theta \mathcal{L}^\top \frac{\partial \theta_*}{\partial \eta}.
\label{eq_ori_grad}
\end{equation}
Technically speaking, this $\mathcal{L}$ can be a validation loss different from the training loss. We consider the training loss for simplicity, as is often supposed in the previous work~\citep{yang2020feature}. The methods presented in this work can also apply to the validation loss if we have validation samples.

Section \ref{sec3} provides a framework for implicitly evaluating Eq.~(\ref{eq_ori_grad}) without directly computing $\partial \theta_*/\partial \eta$. Section \ref{sec4} presents an empirically validated approximation for explicitly evaluating the hypergradient.

\section{EoS-constrained LR gradient}
\label{sec3}

In this section, we focus on learning under the EoS phenomenon, which links the LR and trained model parameters during training. 
The key idea is that 
using this relation as a constraint in the method of Lagrange multipliers, we can obtain a formula which the LR gradient should satisfy.

\subsection{Basic formula}

Suppose that the model is trained under the EoS, and that its global minimum is the same as the solution of the following problem: 
\begin{equation}
\min_\theta \mathcal{L}(\theta) \quad s.t. \quad \eta = F(\theta).
\label{eq_ori_opt}
\end{equation}
$F$ denotes the inverse of the sharpness. We suppose that $\L$ is twice differentiable and $F$ is differentiable.
For gradient descent (GD), EoS holds for the following $F$~\citep{cohen2021gradient}: 
\begin{equation}
F_{naive}(\theta) = 2/\lambda_{max}(H(\theta)), 
\end{equation}
where $H$ denotes the Hessian of the loss function. 
The dependence on the algorithm can be incorporated into $F$. As we show later, several previous studies have examined the algorithm dependence of $\eta=F(\theta)$, and the corresponding expressions can be substituted into our formulation.

The following is a simple idea, but it leads to a novel and interesting relation for the LR gradient.
\begin{lemma}
Consider solving the optimization problem (\ref{eq_ori_opt}) by the Lagrange multiplier formulation $\min_\theta \mathcal{L}(\theta) + \lambda F(\theta)$ for a multiplier $\lambda \in \mathbb{R}$. Then, at the minimum,  
\begin{equation}
d \mathcal{L}/d \eta = - \lambda.
\label{eq_lag}
\end{equation}
\label{lem3_1}
\end{lemma}
{\it Proof.}  We abbreviate $\nabla_\theta$ as $\nabla$. Since $\eta = F(\theta)$, we also have $1=dF/d\eta = \nabla F^\top \partial \theta/\partial \eta$. Combining this with $\nabla \mathcal{L}(\theta) + \lambda \nabla F(\theta)=0$, 
we obtain the result. \qed

We first explain an explicit expression and its computation using the most basic case, $F=F_{naive}$, as a concrete example.
 Taking the inner product of $\nabla\mathcal{L}+\lambda\nabla F=0$ with $\gL$ and using Eq. (\ref{eq_lag}), we obtain the following {\it EoS-LR gradient}: 
\begin{equation}
\qquad\frac{d \mathcal{L}}{d \eta}
 = \frac{\|\gL\|^2}{ \gL^\top \nabla F} %
= -\frac{2\|\gL\|^2}{\eta^2 \gL^\top \nabla \lambda_{max}}.
\label{eq_dL_naive}
\end{equation}

This expression is easy to treat in the sense that the numerator is the squared gradient norm easily obtained for the trained model. However, its denominator involves $\nabla \lambda_{max}$, a quantity based on third derivatives of the loss. This quantity has rarely been examined numerically in previous work, and its stable numerical computation is generally unexplored. We evaluate the directional derivative $\gL^\top\nabla F$ by finite differences, avoiding explicit third derivatives (Appendix~\ref{app:eos-lr-implementation}).
The Lagrange multiplier method only examines how
the constrained minimization of the loss itself constrains
the form of the solution, which does not explicitly depend on learning dynamics.  Nevertheless, our expression explains the LR gradient of the trained models well, as is shown in
Figure \ref{fig:fig1}.
When EoS holds, as shown in Figure~\ref{fig:fig1}(b), substituting the trained model parameters into Eq.~(\ref{eq_dL_naive}) accurately predicts the LR gradient computed by finite differences, as shown in Figure~\ref{fig:fig1}(c). 
We empirically observed that when EoS holds across widths, the  curve transfers across widths (i.e., the position of the U-shaped curve does not shift with width), consistent with the findings of \cite{noci2024super}.

Lemma \ref{lem3_1} and Eq.~(\ref{eq_dL_naive}) imply that, at ideal minima, $\nabla \L^\top\nabla\lambda_{max}$ is positive when the LR curve decreases with $\eta$.
This inner product is sometimes assumed to be negative during the loss decrease along training trajectories in EoS theory~\citep{damian2023self}, but its sign need not be the same at averaged parameters.
The angle may differ between an optimization trajectory exhibiting period-two oscillations and the flat manifold on which our theory assumes the constrained minimum in Eq.~(\ref{eq_ori_opt}).
Consistent with this distinction, we observe a positive inner product at the averaged parameters of trained models (Figure~\ref{fig:alignment-resnet} in Appendix \ref{sec_C1}).
Appendix~\ref{app:eos-lr-implementation} describes how we compute the EoS-LR gradient in Figure~\ref{fig:fig1}.

\begin{figure}
    \centering
    \includegraphics[width=\linewidth]{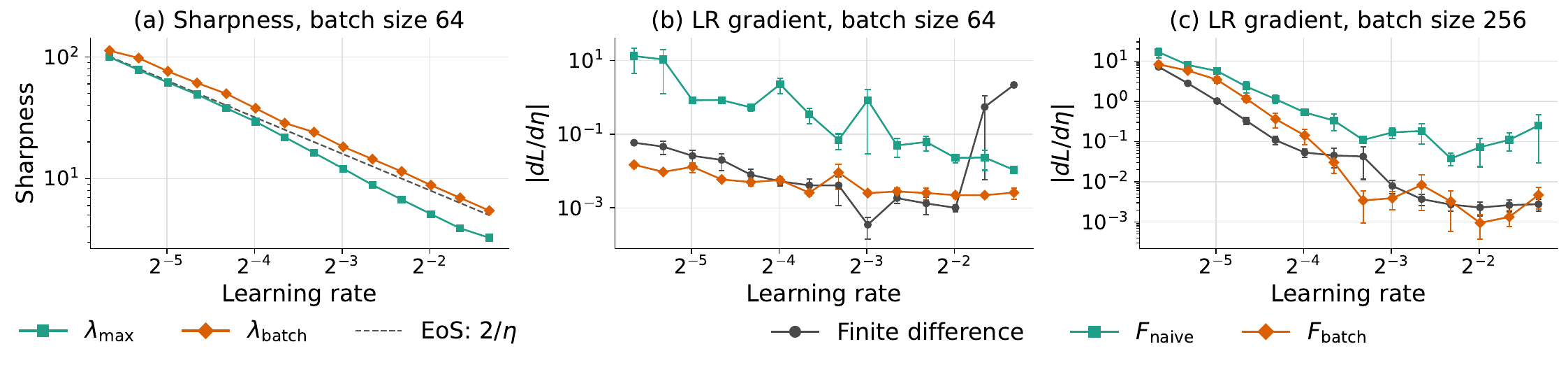}
    \caption{\textbf{EoS-LR gradients based on $F_{batch}$ capture LR sensitivity at small batch sizes.}
    (a) As the batch size decreases, $\lambda_{max}$ deviates from $2/\eta$.
    (b,c) At small batch sizes, the LR gradient based on $F_{batch}$ closely matches the finite-difference estimate.}
    \label{fig:bs-lrgrad-eos}
\end{figure}

\noindent
{\bf Batch Sharpness. }
Similarly, by substituting different $F$ into the intermediate expression in Eq.~(\ref{eq_dL_naive}), we can compute $d\L/d\eta$ for individual EoS cases.
\cite{andreyev2024edge} proposed 
\begin{equation}
F_{batch}(\theta) =2 /E_B[ \gL(B)^\top H(B) \gL(B)/\|\gL(B)\|^2],
\end{equation}
where the expectation $E_B$ is taken over the mini-batches.
For SGD with momentum, $F_{batch}(\theta)$ has been argued to approach $2(1+\beta)/\lambda_{max}$ in the large-batch limit and $2(1-\beta)/\lambda_{max}$ in the small-batch limit \citep{andreyev2026momentum}. 
We examine this $F_{batch}$ as an example showing that $F$ need not be restricted to $F_{naive}$ and can quantitatively explain the LR curves. 
The results are shown in Figure~\ref{fig:bs-lrgrad-eos}.
For small batch sizes, $F_{naive}$ fails to match the finite-difference estimates, whereas $F_{batch}$ successfully matches them.

\section{Fixed-point LR gradient}
\label{sec4}
EoS-LR characterizes LR sensitivity through an EoS constraint, whose form depends on the optimizer. We now introduce a heuristic approximation of the LR gradient, derived directly from the training updates, without requiring an explicit EoS constraint. 
While naive recursive computation of the hypergradient and its fixed-point approximation can still be unstable, we find that a time-averaged fixed-point approximation works well in a stable manner, which is referred to as the FP-LR gradient.
This formulation is more flexible than the EoS-LR gradient, in the sense that it can be adapted to more complex update rules, including AdamW (Appendix~\ref{app:fp-adamw}), and, under a block-diagonal Hessian approximation, decomposed into layer-wise contributions to LR sensitivity (Section~\ref{sec5_layer}).

\begin{figure}
    \centering
    \includegraphics[width=\linewidth]{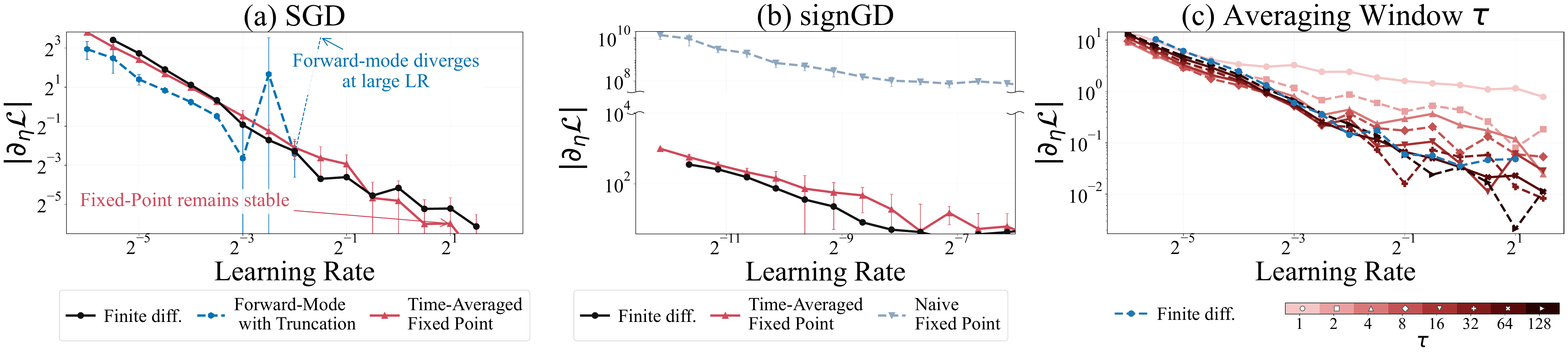}
    \caption{\textbf{FP-LR gradients agree with finite-difference estimates.}
    (a,b) FP-LR captures LR sensitivity for both SGD and signGD.
    (c) Increasing the averaging window improves agreement between the FP-LR gradient and the finite-difference estimate.}
    \label{fig:fp-lr-grad-bound}
\end{figure}

{\bf Naive LR Gradient.}  
As a preliminary step, we return to the LR hypergradient in Eq.~(\ref{eq_ori_grad}). At training step $t$, it is given by $d\mathcal{L}(\theta_t)/d\eta = \nabla\mathcal{L}(\theta_t)^\top \partial\theta_t/\partial\eta$ and can be computed by either reverse-mode or forward-mode differentiation through the training updates~\citep{franceschi2017forward}.
Consider the parameter update rule $\theta_{t+1} = (1 - \eta \gamma)\theta_t - \eta \nabla \mathcal{L}(\theta_t)$, where $\gamma$ is the weight decay coefficient.
Differentiating this update rule with respect to $\eta$ gives the forward-mode recurrence for $\partial\theta_t/\partial\eta$ appearing in Eq.~(\ref{eq_ori_grad}):
\begin{equation}
\partial \theta_{t+1}/\partial \eta = \{I - \eta(H_t + \gamma I)\}\, \partial \theta_{t}/\partial \eta - (\nabla \mathcal{L}(\theta_t) + \gamma \theta_t),
\label{eq:forward-mode-gd}
\end{equation}
where $H_t$ denotes the Hessian of $\mathcal{L}(\theta)$ evaluated at $\theta_t$. If the initial weights are independent of $\eta$, the recurrence starts from $\partial\theta_0/\partial\eta=0$.
For simplicity, we include the $\ell_2$ regularization term in the loss and do not explicitly write the $\gamma$ term below.

As reported in previous work, the naive LR gradient tends to diverge, and we observed the same trend in our experiments~\citep{maclaurin2015gradient,micaelli2021gradient}.
Similarly, we tested the naive fixed-point approximation of Eq.~(\ref{eq:forward-mode-gd}), that is, 
${\partial \theta_*}/{\partial \eta}
=
-{\eta^{-1}}H(\theta_t)^{-1}\gL(\theta_t)$, 
which assumes that $\partial\theta/\partial\eta$ has approximately reached a stationary point.
However, we empirically observed that this approximation was either unstable or differed from the true LR gradient estimated from the finite difference~(\ref{fin_diff}).

{\bf FP-LR gradient.} 
We empirically find that the following time-averaged quantity is computationally stable and provides an accurate approximation to the true LR gradient as a heuristic:
\begin{equation}
    \frac{\partial \theta_*}{\partial \eta}
    =
    -\frac{1}{\eta}\bar{H}_t^{-1}
    \left\langle \gL(\theta_t) \right\rangle,
    \label{eq9:0925}
\end{equation}
where $\bar{H}$ and $\langle \cdot \rangle$ denote time averages over a  finite time window $[t-\tau,t]$.
We then compute $d\L /d \eta =\gL(\langle \theta_t \rangle)^\top \partial \theta_*/\partial \eta$, referred to as the FP-LR gradient.  
Appendix~\ref{app:fp-lr-derivation} gives more details on the numerical computation of time averages and extension to signGD and AdamW.
Note that in practice, neural networks are commonly trained with LR schedules and for a large but finite number of epochs, and thus the averaged gradient is not generally  zero. Therefore, the time-averaged gradient need not vanish even near the end of training.

As shown in Figure~\ref{fig:fp-lr-grad-bound}, the FP-LR gradient agrees with the finite-difference of the LR curve for both SGD and signGD. With a smaller averaging window, the estimate differs more from the finite-difference LR gradient. This difference decreases as the window size ($\tau$) increases. One interpretation is that the time average removes noisy short-timescale  dynamics caused by discrete updates. %
For example, under EoS, the gradient has been reported to exhibit period-two sign oscillations satisfying $\gL(\theta_{t+1}) \approx -\gL(\theta_t)$ \citep{cohen2021gradient,damian2023self}, while  the parameters gradually move toward flatter regions over longer timescales.

\section{Quantitative understanding of LR transfer curves}

In this section, we use the EoS-LR and FP-LR gradients derived above in a complementary manner and demonstrate their utility in understanding the shape of LR curves.

\subsection{Analysis of Scale-invariant models}
\label{sec5_norm}

Parameter blocks routed through a normalization layer, such as  weights immediately preceding a BatchNorm or LayerNorm, or   $W_Q$ and $W_K$ in attention with QK normalization, exhibit scale invariance, i.e., $\mathcal{L}(c\theta) = \mathcal{L}(\theta)$ for any $c > 0$.
Such scale-invariant models have been explored independently in both EoS theory~\citep{arora2022understanding} and empirical studies of better LR transfer~\citep{kosson2025weight}.
Here, we reveal that the LR  gradient becomes zero in this scale-invariant case, implying LR transfer in a more unified manner. 

\paragraph{Homogeneity of $F$.} If the scale-invariant loss function is  twice differentiable, we have $H(c\theta)=c^{-2}H(\theta)$ and thus $\lambda_{max}(H(c\theta))=c^{-2}\lambda_{max}(H(\theta))$.
Therefore, $F_{naive}(c\theta)=c^2 F_{naive}(\theta)$. Similarly, one can immediately check that $F_{batch}$ is homogeneous of degree $2$ as well. In contrast, 
 the EoS conditions of adaptive optimizers  show a different homogeneity. For signGD,
\cite{islamov2026non} show $F_{sign}=2\|\gL\|_1/S_{\infty}$ where $S_{\infty}=\max_{\|d\|_\infty=1} d^\top \nabla^2 \L d$.
This satisfies $F_{sign}(c\theta)=cF_{sign}(\theta)$. 
Note that for Adam, 
\cite{cohen2022adaptive} proposed
$F_{Adam}(\theta) = 2/\left( \frac{1-\beta_1}{1+\beta_1}\lambda_{max}^P  \right)$,
where $\lambda_{max}^P$ is the largest eigenvalue of $P^{-1/2} H P^{-1/2}$ with Adam's diagonal preconditioner $P$. If we neglect the damping term of the denominator, this is also homogeneous of degree $1$ in the same way as in signGD. %
This homogeneity leads to the following.

\begin{figure}
    \centering
    \includegraphics[width=\linewidth]{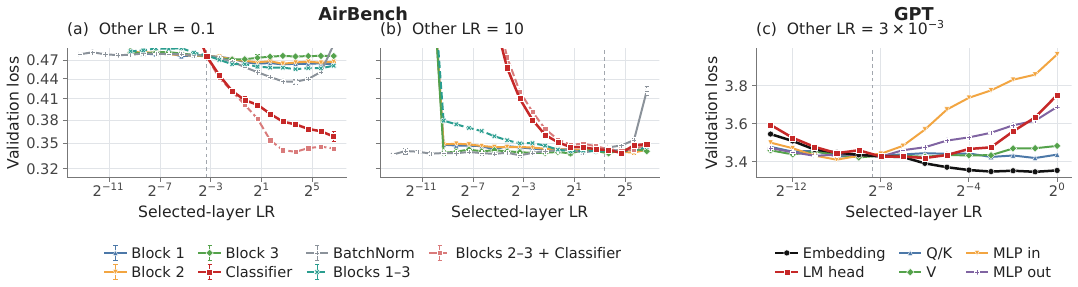}
    \caption{\textbf{Validation loss versus selected-layer LR.} We sweep the LR of each indicated parameter group while fixing the others: (a,b) ResNet on CIFAR-10; (c) GPT on FineWeb.}
    \label{fig:block-lr-resnet-cifar}
\end{figure}

\begin{theorem}
Assume that the model is scale-invariant and $F$ satisfies the homogeneity $F(c\theta)=c^a F(\theta)$  for $a>0$.  
Then, the solution of the minimization problem in Lemma \ref{lem3_1} satisfies
 ${d\L}/{d\eta}=0$.
 \label{thm1}
\end{theorem}
The proof is shown in Appendix \ref{secGD_SI}. The point is that $F$'s homogeneity leads to $\theta^\top \nabla F= a F$. 
This forces the multiplier $-\lambda$ in the naive case (\ref{eq_lag}) to cancel out, yielding a zero LR gradient.
Note that for GD in the scale-invariant model, FP-LR without time averaging is easily shown to be 0 as in Appendix \ref{FPLR_SI}. In comparison, EoS-LR is useful in that it accommodates various optimizers through the choice of $F$.

\noindent
{\bf A distinctive property of Sign gradient with weight decay.} 
Thus far, the EoS-LR formulation has required only  loss minimization. In practice, learning may involve additional explicit or implicit regularization, making its effect on EoS-LR an interesting question.
Recently, \cite{kosson2025weight} empirically reported that the combination of weight decay (L2 regularization) and Adam is a key factor for LR transfer.
More specifically, assuming $\gamma \eta = const.$ for the regularization coefficient $\gamma$ and a gradient norm of order 1, they argue that the parameter norm satisfies $\|\theta\|^2 = C\eta^2$ with an uninteresting constant $C$, which leads to a stable transfer.
Although their derivation of this regularization scale is intuitive, the bound $\|\theta\|^2 \leq C\eta^2$ follows readily from these assumptions, as shown in Lemma \ref{lem2}.
We find that the advantage of this setting can indeed be verified through EoS-LR analysis in the scale-invariant model.
\begin{theorem}
Consider the minimization problem (\ref{eq_ori_opt}) with an L2 constraint  $\|\theta\|^2 \leq C \eta^2$, that is,  $\theta_*=\text{argmin}_\theta \mathcal{L}(\theta)+\lambda_1 F(\theta)+ \lambda_2(\|\theta\|^2-C \eta^2)/2$ with $\lambda_2 \geq 0$. 
Under the same assumption of scale invariance as in Theorem \ref{thm1}, 
\begin{equation}
d\L(\theta_*)/d\eta= (1-a) \|\gL\|^2/(\gL^\top \nabla F). 
\end{equation}
 In particular, for $a=1$, the LR gradient is zero.   
\label{thm2}
\end{theorem}
The proof is shown in Appendix \ref{secGD_WDSI}.
For signGD and simplified Adam, where $a=1$, the LR gradient indeed vanishes, eliminating LR sensitivity.
We can also derive an explicit, though more complicated, expression for the LR gradient with L2 regularization but without scale invariance as in Eq.~(\ref{eq_wo_norm}).
Thus, the LR gradient is generally nonzero if either L2 regularization or scale invariance is absent.

In ResNet (Figure~\ref{fig:block-lr-resnet-cifar}), every layer except the final layer is immediately followed by normalization.
Thus, all layers except the final layer and the BatchNorm layers are scale-invariant, and the loss is not sensitive to their learning rates.
In GPT, the QK layers are immediately followed by QK normalization.
For the other layers, normalization is applied after the residual connection, so they are not fully scale-invariant.
In particular, the MLP input layer is not immediately followed by normalization and is therefore highly sensitive to changes in its learning rate.

\subsection{Bounds on LR gradients}
\label{sec5_bounds}

Figure~\ref{fig:lrgrad-exponents} shows that LR gradient amplitudes across widths of CNNs, ResNets, and GPT scale approximately as $\eta^{-3/2}$ over a broad LR range, generally between $\eta^{-1}$ and $\eta^{-2}$.
A natural question is whether our EoS-LR or FP-LR can explain these empirical slopes.
However, these expressions include the gradient norm and higher-order derivatives, which depend implicitly on the LR, making these slopes nontrivial to determine.
Here, we establish consistency with the experimental results by evaluating bounds under several assumptions.

\paragraph{Case of EoS-LR.}
The Cauchy--Schwarz inequality straightforwardly gives
\[
|d\L/d\eta| \geq 2\|\gL\|/(\eta^2 \|\nabla \lambda_{max}\|).
\]
(i) Since both theoretical and empirical studies of $\|\nabla \lambda_{max}\|$ are limited in the literature, we evaluated it numerically and observed scaling of order $\eta^{-1}$, as shown in Fig~\ref{fig:lambda_grad_lr_exponents}.
(ii) For the gradient norm, a first-order expansion of the loss update gives
$\mathcal{L}_{t+1} \approx \mathcal{L}_{t} - \eta \|\gL_t\|^2$.
Assuming that training progresses at order one, as in~\cite{yang2020feature,chizat2024feature}, we have $\|\gL_t\|^2=C/\eta$.
The same scaling has also been reported in  experiments on implicit gradient regularization~\citep{barrett2020implicit}, supporting this assumption.
Combining (i) and (ii) yields $|d\L/d\eta| = \Omega(\eta^{-3/2})$, consistent with Figure \ref{fig:lrgrad-exponents}.

\paragraph{Comparison with FP-LR.}
When the training loss is sufficiently small, the Hessian is expected to be dominated by its Gauss--Newton part, motivating the assumption $H\succeq0$, which gives
$\|\gL\|^2/\eta(\lambda_{max}+\gamma) \leq
|d\L/d\eta| \leq \|\gL\|^2/\eta\gamma$.
Depending on additional assumptions, these upper and lower bounds can scale as either $1/\eta$ or $1/\eta^2$.
First, assuming $\eta \|\gL\|^2 = const.$ as in the above EoS case, both bounds scale as $1/\eta^2$.
If we further assume EoS, the lower bound scales as $1/\eta$.
If we impose $\eta \gamma = const.$~\citep{kosson2025weight}, the upper bound also scales as $1/\eta$.
Thus, the range of bounds for FP-LR is consistent with the experimental results, but the bounds themselves do not scale as $\eta^{-3/2}$.
In this sense, EoS-LR provides a more specific quantitative characterization.

\begin{figure}
    \centering
    \includegraphics[width=0.95\linewidth]{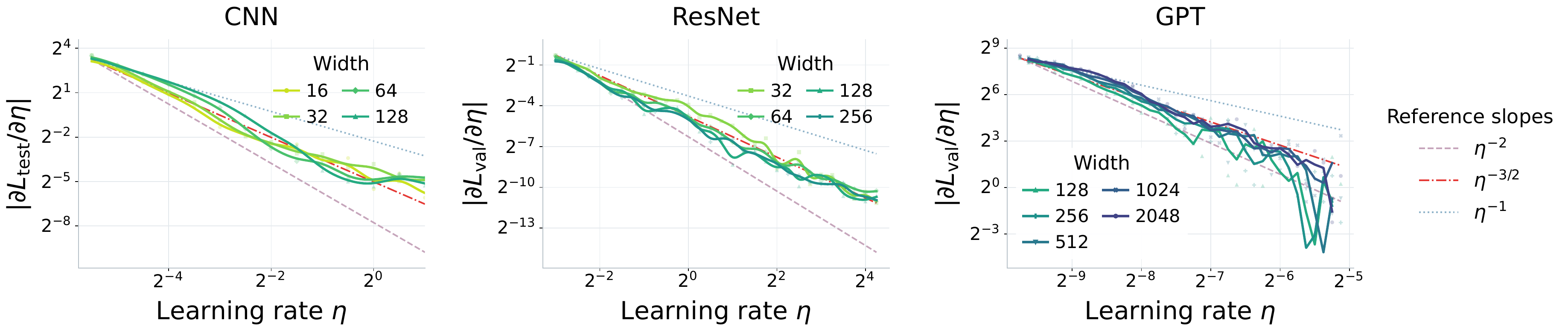}
    \caption{\textbf{LR gradients of the evaluation loss exhibit $\eta^{-3/2}$ scaling.}
    Log--log plots show the absolute LR gradients for CNN, ResNet, and GPT, estimated by finite differences from LR sweeps.
    We use the test loss for CNN and the validation loss for ResNet and GPT; colors indicate network width.
    Reference lines show the slopes $\eta^{-2}$, $\eta^{-3/2}$, and $\eta^{-1}$ suggested by the bounds.}
    \label{fig:lrgrad-exponents}
\end{figure}

\subsection{Layerwise Profiling}
\label{sec5_layer}

\begin{figure}
    \centering
    \includegraphics[width=\linewidth]{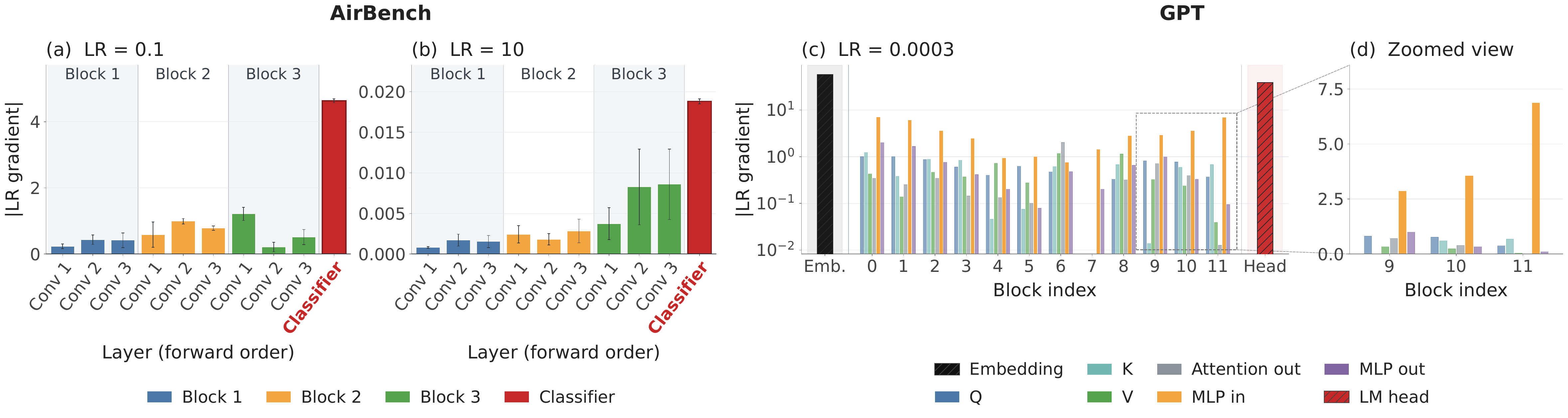}
    \caption{\textbf{Layerwise FP-LR gradient.} (a,b) ResNet on CIFAR-10 with two different LRs. (c) GPT on FineWeb. (d) Blocks $9$--$11$ of (c), shown on a linear scale.}
    \label{fig:gpt-qknorm-lrgrad-profile}
\end{figure}

Motivated by the benefits of layer-wise LRs~\citep{yang2021tuning, he2026one}, we use the FP-LR expression to decompose the LR gradient into layer-wise contributions and identify the layers that contribute most to LR sensitivity.
Under a layer-wise block-diagonal approximation to the time-averaged Hessian, the FP-LR gradient decomposes into layer-wise components as
\begin{equation}
\frac{d\L}{d\eta}= -\sum_l \frac{1}{\eta}\nabla_{\theta^{(l)}}\L^\top (\bar H^{(l)})^{-1}\langle\nabla_{\theta^{(l)}}\L\rangle
= \left.\sum_l\frac{\partial\L}{\partial\eta_l}\right|_{\eta_l=\eta},
\label{eq:layer-wise-fp-lr}
\end{equation}
where $\theta^{(l)}$ and $\eta_l$ are the parameters and LR of layer $l$, with update $\theta_{t+1}^{(l)}=\theta_t^{(l)}-\eta_l\nabla_{\theta^{(l)}}\L(\theta_t)$, and $\bar H^{(l)}=\nabla_{\theta^{(l)}}^2\L(\langle\theta_t\rangle)$. Time averages and gradient evaluation follow Section~\ref{sec4}; implementation details are in Appendix~\ref{app:layer-wise-kfac}.
Note that this simple decomposition is an advantage of FP-LR, compared to  EoS-LR. EoS-LR includes the inner product of all parameters in both numerator and denominator in Eq. (\ref{eq_dL_naive}) and this makes the explicit decomposition more non-trivial. 

As is shown in Figure~\ref{fig:gpt-qknorm-lrgrad-profile}(a,b), for ResNet, the final layer (Classifier) has a larger amplitude of the layer-wise FP-LR gradient than the other layers.
This result is consistent with Figure~\ref{fig:block-lr-resnet-cifar}, which shows that the loss varies substantially with the LR of the final layer but only slightly with that of the input layer.
For GPT, the embedding and output layers have  particularly large layer-wise LR gradients, and the MLP input layers also have large LR gradients within the Transformer blocks.
These profiles provide a quantitative measure of each layer's contribution to LR sensitivity.

\begin{figure}
    \centering
    \includegraphics[width=\linewidth]{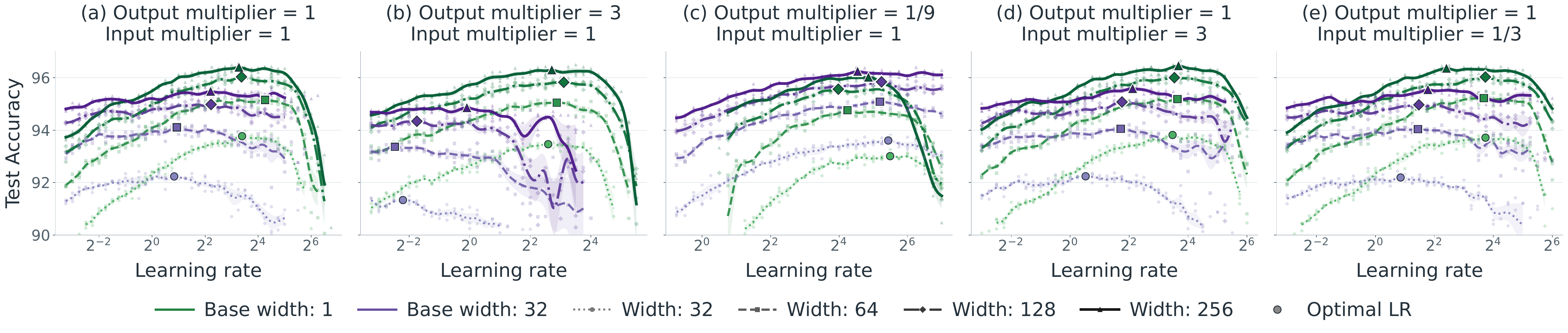}
    \caption{
    \textbf{Accuracy and LR transfer under different input and output multipliers.}
    Test accuracy for ResNet on CIFAR-10 under $\mu$P, comparing Base Width=$1$ and $32$.
    (a) Both multipliers are $1$.
    (b,c) Only the output multiplier changes to $3$ or $1/9$, respectively.
    (d,e) Only the input multiplier changes to $3$ or $1/3$, respectively.
    }
    \vspace{-10pt}
    \label{fig:base_width_resnet_cifar}
\end{figure}

Figure~\ref{fig:base_width_resnet_cifar} demonstrates that LR transfer across widths is also sensitive to the output-layer scale, as suggested by Figure~\ref{fig:gpt-qknorm-lrgrad-profile}. 
Here, the input and output multipliers rescale the first layer's output and the final logits, respectively. Base Width is the reference width for increasing the width. These quantities are commonly used in $\mu$P studies (see Appendix~\ref{app:base-width-multipliers} for details). Base Width is often set to the smallest width used in experiments, and this motivates the comparison with Base Width=$1$, which rigorously obeys the $\mu$P scaling.
With both multipliers at $1$, Base Width=$1$ gives higher accuracy than Base Width=$32$ (a).
Reducing the output multiplier to $1/9$ improves accuracy at Base Width=$32$ to roughly the same level as Base Width=$1$ with default multipliers (c). Increasing it to $3$ instead lowers accuracy (b).
At Base Width=$1$, reducing the output multiplier lowers accuracy and makes the optimal LR differ across widths, making LR transfer less reliable (c).
Finally, changing the input multiplier has little effect (d,e).
Thus, we can conclude that LR transfer significantly depends on the output-layer scale and favors scaling the output layer according to $\mu$P from Base Width=$1$.

\section{Limitations and Discussions}

We provide a unified perspective on the LR curve from the analyses of the LR gradient and expect that our results will serve as a foundation for better understanding the mechanisms of stable LR tuning.
 The following directions could be important for further development. (i) Our characterization of the LR gradient under EoS via Lagrange multipliers will be broadly applicable. That is, as future theoretical or empirical analyses reveal new EoS relations in different optimizers, these can be substituted into our formulation to derive the LR gradient. For example, EoS for preconditioned optimizers in more practical settings incorporating moving averages or damping remains unresolved in the EoS literature~\citep{islamov2026non}. Clarifying these relations would also enable the characterization of LR curves through our EoS-LR gradient. (ii) Since this work is restricted to evaluating first derivatives of the LR curve,  evaluating higher-order derivatives is also an interesting direction for understanding the shape of the curves. However, it is noteworthy that in both the EoS and FP formulations, additional higher-order terms arise as the derivative order increases. The resulting expressions are therefore unlikely to be as simple as those obtained here. Even the present expressions require careful numerical computation.    (iii) We showed that the LR gradient vanishes for scale-invariant layers, yielding a width-independent result. Analyzing the model-size dependence of our relational equations in more general cases is an interesting theoretical direction. Such an analysis requires understanding the universality of the derivative of the largest eigenvalue, which involves third derivatives of the loss. While the theory of the large-size limit of such higher-order terms remains under development~\citep{noci2024super,lauditi2026spectral}, our work provides a foundation for translating future theoretical advances into evaluations of the LR gradient.

\section*{Acknowledgements}

The authors acknowledge support from JST FOREST (Grant No. JPMJFR226Q) and FRONTia, a Japanese national program led by METI and NEDO. RK is also supported by JST BOOST (Grant No. JPMJBY25B0), and SI by JSPS KAKENHI (Grant No. 26KJ1108) and JST ACT-X (Grant No. JPMJAX25C1).

\newpage
\bibliographystyle{plainnat}
\bibliography{reference}

\newpage
\appendix
\raggedbottom
\setlength{\floatsep}{8pt plus 2pt minus 2pt}
\setlength{\textfloatsep}{10pt plus 2pt minus 2pt}
\setlength{\intextsep}{6pt plus 2pt minus 2pt}
\renewcommand{\floatpagefraction}{0.8}
\makeatletter
\setlength{\@fptop}{0pt}
\setlength{\@fpsep}{10pt plus 2pt minus 2pt}
\setlength{\@fpbot}{0pt plus 1fil}
\makeatother
\part*{Appendices}

\section{Derivations of LR gradients}

\subsection{Scale invariance}
\label{secGD_SI}

\begin{theorem}[Theorem \ref{thm1}]
Assume scale invariance $\L(c \theta)= \L(\theta)$ and $F(c\theta)= c^a F(\theta)$  ($c,a>0$). Consider the minimization of $\L(\theta)$ with a constraint $\eta = F(\theta)$, and its formulation with a Lagrange multiplier:
\begin{equation}
\theta_* :=\text{argmin}_\theta \L(\theta) + \lambda F. 
\end{equation}
 Then,
\begin{equation}
\frac{d\L(\theta_*)}{d\eta}=0.
\end{equation}
\end{theorem}
{\it Proof.} As is shown in Lemma \ref{lem3_1},
\begin{equation}
\frac{d\L}{d\eta} = -\lambda.
\end{equation}
Here, note that one can easily confirm 
\begin{equation}
\theta^\top \nabla\L(\theta) =0
\label{eq_thetaL}
\end{equation}
by differentiating $\L(c \theta)= \L(\theta)$ with respect to $c$. Similarly, we have
\begin{equation}
 \theta^\top \nabla F(\theta) = a F
\label{eq_thetaF}
\end{equation}
by differentiating $F(c \theta)= c^a F(\theta)$ with respect to $c$.
Taking the inner product of both sides of  $\gL + \lambda \nabla F=0$ with $\theta$, we have
\begin{equation}
\lambda \theta^\top \nabla F= \lambda a \eta =0.
\end{equation}
Since $\eta>0$ and $a>0$, we have $\lambda=0$ and 
\begin{equation}
\frac{dL}{d\eta}=0.
\end{equation}
\qed

\subsection{ Scale invariance with weight decay}
\label{secGD_WDSI}

\begin{lemma}
Consider a gradient update
$\theta_{t+1}= (1- \eta \lambda) \theta_t - \eta  g(\theta_t)$. Assume (i) $\|g(\theta_t)\| \leq C_0$ for all $t\geq 0$ for a constant $C_0>0$
and (ii) a constant $\eta \lambda = k$ ($0<k<1$). 
Then, 
\begin{equation}
\|\theta_{t+1}\|^2 \leq   C \eta^2
\end{equation}
for a constant $C>0$ and sufficiently large $t$.
\label{lem2}
\end{lemma}

{\it Proof.} 
We have 
\begin{equation}
\|\theta_{t+1}\|\leq (1- \eta \lambda) \|\theta_t\|+ \eta \|g_t\|.
\end{equation}
Therefore, 
\begin{align}
\|\theta_{T+1}\| &\leq (1-\eta\lambda)^{T+1}\|\theta_0\|
+ \sum_{j=0}^T \eta (1-\eta\lambda)^j \|g_{T-j}\| \\
&\leq (1-\eta\lambda)^{T+1}\|\theta_0\|
+ \eta C_0 \frac{1-(1-\eta \lambda)^{T+1}}{\eta \lambda}.
\end{align}

Choose an integer $T\geq 0$ such that $T > {\ln\!\left(\frac{C_0/\lambda}{\|\theta_0\|+C_0/\lambda}\right)}/{\ln(1-\eta\lambda)}$.
Since $0<1-\eta\lambda<1$, this ensures
$(1-\eta\lambda)^{T+1}(\|\theta_0\|+C_0/\lambda)\leq C_0/\lambda$.
The preceding bound therefore gives
\begin{equation}
\|\theta_{T+1}\|
\leq (1-\eta\lambda)^{T+1}(\|\theta_0\|+C_0/\lambda)+C_0/\lambda
\leq 2C_0/\lambda.
\end{equation}
In addition, by using the assumption $\eta \lambda = k$, 
\begin{equation}
\|\theta_{T+1}\|^2 \leq 4C_0^2  \eta^2/k^2.
\end{equation} 
Then, setting $C:=4C_0^2/k^2$, we obtain, for all $t\geq T$,
\begin{equation}
\|\theta_{t+1}\|^2 \leq C  \eta^2.
\end{equation}
\qed

\cite{kosson2025weight} derived an equality for
$\theta_{t+1}=\theta_t-\eta g(\theta_t)$
by assuming $\theta_t^\top\nabla g(\theta_t)=0$,
but this assumption is nontrivial except for ordinary GD.
Our assumption is more general in that it does not explicitly depend on the optimizer.

As shown above, assuming $\eta\lambda=k$, bounded gradients yield, for sufficiently large $t$,
$\|\theta\|^2\leq C\eta^2$.
Thus, it is natural for the L2 norm of the parameters to be bounded
by a constant times $\eta$.
We therefore consider imposing
$\|\theta\|^2\leq C\eta^2$
in our Lagrange multiplier method as a minimal constraint
that the converged parameters should satisfy.
Then, for scale-invariant layers, we obtain the following:
\begin{theorem}[Theorem \ref{thm2}]
Assume scale invariance $\L(c \theta)= \L(\theta)$ and $F(c\theta)= c^a F(\theta)$ ($c,a>0$). Consider the minimization of $\L(\theta)$ with constraints $\eta=F(\theta)$ and $\|\theta\|^2 \leq C \eta^2$:
\begin{equation}
\theta_*= \text{argmin}_\theta \L(\theta) + \lambda_1 F + \frac{\lambda_2}{2} (\|\theta\|^2- C \eta^2)
\end{equation}
with Karush–Kuhn–Tucker (KKT) multipliers $\lambda_1$ and $\lambda_2 \geq 0$. Then,
\begin{equation}
\frac{d\L(\theta_*)}{d\eta}=(1-a) \frac{\|\gL\|^2}{ \gL^\top \nabla F}.
\end{equation}
In particular, for $a=1$, 
\begin{equation}
 \frac{d\L(\theta_*)}{d\eta}=0.  
\end{equation}
\end{theorem}
{\it Proof.}
The KKT conditions of this constrained minimization problem are given by 
\begin{equation}
\gL + \lambda_1 \nabla F +  \lambda_2 \theta=0,
\label{eq_fix_WD}
\end{equation}
\begin{equation}
 \frac{\lambda_2}{2} (\|\theta\|^2- C \eta^2)=0.
 \label{eq_KKT_slack}
\end{equation}

Case (i): $\lambda_2=0$. 
Taking the inner product of both sides of  Eq.~(\ref{eq_fix_WD}) with $\theta$, we have
\begin{equation}
\theta^\top (\gL + \lambda_1 \nabla F) = \lambda_1 (a \eta) = 0,
\end{equation}
because of $\theta^\top \gL(\theta)=0$ (eq. \ref{eq_thetaL}) from the scale invariance.
Therefore, $\lambda_1=0$ and we obtain $\frac{d\L}{d\eta}=0$.

Case (ii): $\lambda_2 \neq 0$.  We have $\|\theta\|^2- C \eta^2=0$, and thus,
\begin{equation}
\theta^\top \frac{\partial \theta}{\partial \eta}=C \eta.
\end{equation}
Then, 
\begin{align}
\theta^\top (\gL + \lambda_1 \nabla F +  \lambda_2 \theta) &= \lambda_1 a \eta + \lambda_2 \|\theta \|^2 \\
&= \lambda_1 a \eta + \lambda_2 C \eta^2 \\
&=0,
\end{align}
where we used eqs.(\ref{eq_thetaL},\ref{eq_thetaF}) in the first line and the KKT condition (\ref{eq_KKT_slack}) in the second line. 
We have  
\begin{equation}
\lambda_2 = - \frac{a}{C \eta} \lambda_1.
\end{equation}
Finally, 
\begin{align}
\frac{d\L}{d\eta} &=  - \lambda_1  - \lambda_2 \theta^\top \frac{\partial \theta}{\partial \eta} \\
&= -\lambda_1  + \frac{a}{C\eta} \lambda_1 \cdot C \eta  \\
&= (a-1) \lambda_1.
\end{align}
Taking the inner product of both sides of (\ref{eq_fix_WD}) with $\gL$, we have $\lambda_1 = - \|\gL\|^2/ \gL^\top \nabla F$.
\qed

\noindent
{\bf Remark on L2 regularization without the scale invariance.} In this case, the expression of $d\L/d\eta$ gets highly complicated. The constrained minimization gives
\begin{equation}
\frac{d \L}{d\eta} = - \lambda_1 - C \eta \lambda_2
\label{eq_wo_norm}
\end{equation}
with $\gL + \lambda_1 \nabla F + \lambda_2 \theta=0$. Taking the inner product of both sides of this equation with $\theta$ or $\gL$, we obtain the simultaneous equations for ($\lambda_1,\lambda_2$) and their solutions are
\begin{equation}
\lambda_1=\frac{\|\theta\|^2 \|\gL\|^2-(\theta^\top \gL)^2}{z}, \quad \lambda_2=\frac{(\theta^\top \gL) (\gL^\top \nabla F)-(\theta^\top \nabla F) \|\gL\|^2}{z},
\end{equation}
\begin{equation}
z= \theta^\top \gL  \theta^\top \nabla F - \|\theta\|^2 \gL^\top \nabla F.
\end{equation}
One can easily check that if we substitute $\theta^\top \gL=0$, which follows from  scale invariance, into this expression, we recover $d\L/d\eta=(1-a) {\|\gL\|^2}/{\gL^\top \nabla F}$.

\subsection{FP-LR with scale invariance}
\label{FPLR_SI}

Differentiating $\mathcal{L}(c\theta) = \mathcal{L}(\theta)$ with respect to $c$ at $c=1$ gives $\theta^\top\nabla\mathcal{L}(\theta)=0$. Differentiating this identity with respect to $\theta$ gives $\nabla\mathcal{L}+H\theta=0$, hence
\begin{equation}
 \quad H \theta = -\nabla \mathcal{L}.
 \label{eq_theta2}
\end{equation}
Suppose that we have the fixed point of the LR gradient:
\begin{equation}
\frac{\partial \theta}{\partial \eta} = (I - \eta H) \frac{\partial \theta}{\partial \eta} - \nabla \mathcal{L}.
\end{equation}
By multiplying ${\theta}$ by both sides, we have
\begin{equation}
\theta^\top \frac{\partial \theta}{\partial \eta}= {\theta}^\top \frac{\partial \theta}{\partial \eta} - \eta {\theta}^\top  H  \frac{\partial \theta}{\partial \eta} - {\theta}^\top \nabla \mathcal{L}.
\end{equation}
Because of ${\theta}^\top \nabla \mathcal{L} =0$ and 
\begin{equation}
-{\theta}^\top  H \frac{\partial \theta}{\partial \eta} =  \gL^\top  \frac{\partial \theta}{\partial \eta} = \frac{\partial \L}{\partial \eta}
\end{equation}
from eq.~(\ref{eq_theta2}), we obtain
\begin{equation}
\frac{\partial \L}{\partial \eta}=0.
\end{equation}

\section{Implementation}

\subsection{FP-LR Gradient}
\label{app:fp-lr-derivation}

\subsubsection{SGD} 
The FP-LR expression in the GD case has been shown in Eq.~(\ref{eq9:0925}). This expression is a heuristic that empirically explains the experimental results well, and its theoretical justification or derivation is beyond the scope of this work. However,
 in preparation for introducing the cases of other optimizers, let us explain here how the FP-LR expression can be interpreted or informally ``derived'' as a fixed point of the recurrence (\ref{eq:forward-mode-gd}). 

Let $\mathcal{L}$ include the L2 regularization term and write $H_t=\nabla^2\mathcal{L}(\theta_t)$.
For readability, we use $\partial_\eta$ to denote $\partial/\partial\eta$ in the appendix. 
We denote the average from step $t-\tau+1$ to the current step $t$ by
 $\langle f_t\rangle = \tau^{-1}\sum_{s=t-\tau+1}^{t}f_s$.
We average Eq.~(\ref{eq:forward-mode-gd}) and obtain
\begin{equation}
\langle\partial_\eta\theta_{t+1} \rangle-\langle \partial_\eta\theta_t \rangle
 =-\langle\nabla\mathcal{L}(\theta_t)\rangle
  -\eta\langle H_t\partial_\eta\theta_t\rangle.
\end{equation}
Under the assumption of $\langle\partial_\eta\theta_{t+1} \rangle=\langle \partial_\eta\theta_t \rangle$, rearranging then gives
\begin{equation}
 \eta\langle H_t\partial_\eta\theta_t\rangle
 =-\langle\nabla\mathcal{L}(\theta_t)\rangle.
 \label{eq:fp-gd-average-balance}
\end{equation}
Next, we impose the independence assumption $\langle H_t\partial_\eta\theta_t\rangle=\langle H_t\rangle\langle\partial_\eta\theta_t\rangle$ under time averaging.
Substituting this relation into Eq.~(\ref{eq:fp-gd-average-balance}) gives
\begin{equation}
 \eta\langle H_t\rangle\langle\partial_\eta\theta_t\rangle
 =-\langle\nabla\mathcal{L}(\theta_t)\rangle.
\end{equation}
Assuming that $\langle H_t\rangle$ is invertible, we solve for $\langle\partial_\eta\theta_t\rangle$:
\begin{equation}
 \langle\partial_\eta\theta_t\rangle
 =-\frac{1}{\eta}\langle H_t\rangle^{-1}\langle\nabla\mathcal{L}(\theta_t)\rangle.
\end{equation}
In practice, to reduce the cost of curvature computation, it will be rational to replace $\langle H_t\rangle$ by $\bar H_t=\nabla^2\mathcal{L}(\langle\theta_t\rangle)$ as an approximation, which is the Hessian evaluated at the averaged parameters. 

Finally, take the LR gradient of the evaluation loss at the averaged parameters as
\begin{equation}
 \begin{aligned}
 \frac{d\mathcal{L}_{\rm eval}(\langle\theta_t\rangle)}{d\eta}
 &=\nabla\mathcal{L}_{\rm eval}(\langle\theta_t\rangle)^\top\partial_\eta\langle\theta_t\rangle \\
 &=\nabla\mathcal{L}_{\rm eval}(\langle\theta_t\rangle)^\top
 \langle\partial_\eta\theta_t\rangle \\
 &=-\frac{1}{\eta}\nabla\mathcal{L}_{\rm eval}(\langle\theta_t\rangle)^\top
 \bar{H}_t^{-1}\langle\nabla\mathcal{L}(\theta_t)\rangle.
 \end{aligned}
\end{equation}
Thus, we can find the eventual FP-LR expression through some informal but rational operations.

\subsubsection{signGD}
In adaptive optimizers, we express $\mathcal{L}$ as the original loss and  make 
weight decay written explicitly. Consider
\begin{equation}
 \theta_{t+1}=\theta_t-\eta\{\phi(\nabla\mathcal{L}(\theta_t))+\gamma\theta_t\},
 \qquad
 \phi(x)=\frac{x}{|x|+\varepsilon},
\end{equation}
where $\varepsilon>0$, $\phi$ acts coordinatewise and $\gamma$ is held fixed when differentiating with respect to $\eta$. With $H_t=H(\theta_t)$, the exact sensitivity recurrence is
\begin{equation}
 \partial_\eta\theta_{t+1}=\{I-\eta(D_tH_t+\gamma I)\}\partial_\eta\theta_t
            -\{\phi(\nabla\mathcal{L}(\theta_t))+\gamma\theta_t\},
 \label{eq:signgd-recurrence}
\end{equation}
\begin{equation}
 (D_t)_{ij}:=\frac{\varepsilon}{(|\nabla_i\mathcal{L}(\theta_t)|+\varepsilon)^2}\delta_{ij}.
\end{equation}
\paragraph{Naive fixed point.} Freezing the coefficients and imposing a stationary sensitivity gives
\begin{equation}
 \partial_\eta\theta=-\frac{1}{\eta}(DH+\gamma I)^{-1}
                      \{\phi(\nabla\mathcal{L})+\gamma\theta\},
\end{equation}
provided $DH+\gamma I$ is invertible. Without weight decay, this reduces to $\partial_\eta\theta=-\eta^{-1}(DH)^{-1}\phi(\nabla\mathcal{L})$; both the update direction $\phi(\nabla\mathcal{L})$ and the order of the matrix inverse are essential. Under this stationary-sensitivity assumption, the corresponding loss gradient is
\begin{equation}
 \frac{d\mathcal{L}(\theta_t)}{d\eta}
 =-\frac{1}{\eta}\nabla\mathcal{L}(\theta_t)^\top(D_tH_t+\gamma I)^{-1}
                  \{\phi(\nabla\mathcal{L}(\theta_t))+\gamma\theta_t\}.
 \label{eq:sign-gd-micro-fixed}
\end{equation}
For exact signGD, the derivative of the sign function is zero away from zero and undefined at zero, so this inverse formula cannot be obtained by simply setting $\varepsilon=0$.

\paragraph{Time averaging.} Time averaging plays a role in avoiding the above subtlety of $\varepsilon=0$. 
Averaging Eq.~(\ref{eq:signgd-recurrence}) gives
\begin{equation}
 \langle \partial_\eta\theta_{t+1} \rangle -\langle \partial_\eta\theta_t \rangle
 =-\langle\phi(\nabla\mathcal{L}(\theta_t))\rangle-\gamma\langle\theta_t\rangle
  -\eta\langle D_tH_t\partial_\eta\theta_t\rangle
  -\eta\gamma\langle\partial_\eta\theta_t\rangle.
\end{equation}
We first assume that $\langle \partial_\eta\theta_{t+1} \rangle =\langle \partial_\eta\theta_t \rangle$. Rearranging gives
\begin{equation}
 \eta\bigl\{\langle D_tH_t\partial_\eta\theta_t\rangle
       +\gamma\langle\partial_\eta\theta_t\rangle\bigr\}
 =-\bigl\{\langle\phi(\nabla\mathcal{L}(\theta_t))\rangle
       +\gamma\langle\theta_t\rangle\bigr\}.
\end{equation}
Next, we impose the independence assumption $\langle D_tH_t\partial_\eta\theta_t\rangle=\langle D_t\rangle\langle H_t\rangle\langle\partial_\eta\theta_t\rangle$. Substituting this factorization gives
\begin{equation}
 \eta\bigl(\langle D_t\rangle\langle H_t\rangle+\gamma I\bigr)
 \langle\partial_\eta\theta_t\rangle
 =-\bigl\{\langle\phi(\nabla\mathcal{L}(\theta_t))\rangle
       +\gamma\langle\theta_t\rangle\bigr\}.
\end{equation}
As in SGD, we assume $\langle H_t\rangle=\bar H$ to reduce the cost of curvature computation. Substituting $\bar H$ and assuming $\langle D_t\rangle\bar H+\gamma I$ is invertible, we obtain
\begin{equation}
 \langle\partial_\eta\theta_t\rangle=-\frac{1}{\eta}
 (\langle D_t\rangle\bar H+\gamma I)^{-1}
 \bigl\{\langle\phi(\nabla\mathcal{L}(\theta_t))\rangle+\gamma\langle\theta_t\rangle\bigr\}.
\end{equation}
Finally, substituting this averaged sensitivity into the LR gradient at the averaged parameters defined above gives
\begin{equation}
\frac{d\mathcal{L}_{\rm eval}}{d \eta}
 =-\frac{1}{\eta}\nabla\mathcal{L}_{\rm eval}(\langle\theta_t\rangle)^\top
 (\langle D_t\rangle\bar H+\gamma I)^{-1}
 \bigl\{\langle\phi(\nabla\mathcal{L}(\theta_t))\rangle+\gamma\langle\theta_t\rangle\bigr\}.
\end{equation}
This is the FP-LR gradient for signGD which we used in the experiments. 
Note that taking $\phi(x)=x$ and $D_t=I$ recovers the GD case.

To approximate $\langle D_t\rangle$, we use a regularized regression slope for each coordinate over the final $\tau$ steps:
\begin{equation}
 \bar d_i:=
 \frac{\operatorname{Cov}_{s=t-\tau+1,\ldots,t}(\nabla_i\mathcal{L}(\theta_s),\phi(\nabla_i\mathcal{L}(\theta_s)))}
      {\operatorname{Var}_{s=t-\tau+1,\ldots,t}(\nabla_i\mathcal{L}(\theta_s))+\rho},
 \qquad \bar D:=\operatorname{diag}(\bar d_i),
\end{equation}
for $\rho>0$. 
For a Gaussian random variable $G_i$ with positive variance, Stein's identity gives
\begin{equation}
 \operatorname{Cov}(G_i,\phi(G_i))
 =\operatorname{Var}(G_i)\,\mathbb{E}[\phi'(G_i)].
\end{equation}
Consequently, if the gradient fluctuations in each coordinate are approximately Gaussian, the regularized population regression slope approaches the expected derivative $\mathbb{E}[\phi'(G_i)]$ as $\rho\to0$.
This expected derivative corresponds to the $i$-th diagonal entry of $\langle D_t\rangle$.
This motivates using the empirical slopes $\bar d_i$ to approximate $\langle D_t\rangle$ by $\bar D$ in the FP-LR gradient above.

\subsubsection{AdamW}
\label{app:fp-adamw}
Let $m_t=\beta_1m_{t-1}+(1-\beta_1)\nabla\mathcal{L}(\theta_t)$ and $v_t=\beta_2v_{t-1}+(1-\beta_2)[\nabla\mathcal{L}(\theta_t)]^{\odot2}$. Omitting bias correction, AdamW uses $u_t=m_t/(\sqrt{v_t}+\epsilon_{\rm A})$ and $\theta_{t+1}=(1-\eta\gamma)\theta_t-\eta u_t$, with coordinate-wise operations. As in SGD, $\langle\partial_\eta\theta_{t+1}\rangle=\langle\partial_\eta\theta_t\rangle$ gives
\begin{equation}
 \eta\bigl\{\langle\partial_\eta u_t\rangle
       +\gamma\langle\partial_\eta\theta_t\rangle\bigr\}
 =-\bigl(\langle u_t\rangle+\gamma\langle\theta_t\rangle\bigr).
\end{equation}
To express $\langle\partial_\eta u_t\rangle$ in terms of $\langle\partial_\eta\theta_t\rangle$, we fit a diagonal regression slope over the final $\tau$ steps:
\begin{equation}
 \bar a_i=\frac{\operatorname{Cov}_{s=t-\tau+1,\ldots,t}(m_{s,i},u_{s,i})}
 {\operatorname{Var}_{s=t-\tau+1,\ldots,t}(m_{s,i})+\rho},
 \qquad \bar A=\operatorname{diag}(\bar a_i),
\end{equation}
for $\rho>0$. We assume $\langle\partial_\eta u_t\rangle=\bar A\langle H_t\partial_\eta\theta_t\rangle$ as a heuristic replacement for differentiating the full moment dynamics. The same independence and curvature-replacement assumptions as in SGD then give $\langle\partial_\eta u_t\rangle=\bar A\bar H\langle\partial_\eta\theta_t\rangle$. Substituting this relation into the averaged equation and solving for $\langle\partial_\eta\theta_t\rangle$, we arrive at the FP-LR gradient of AdamW that we used: 
\begin{equation}
\frac{d\mathcal{L}_{\rm eval}}{d \eta}
 =-\frac{1}{\eta}\nabla\mathcal{L}_{\rm eval}(\langle\theta_t\rangle)^\top
 (\bar A\bar H+\gamma I)^{-1}
 \bigl(\langle u_t\rangle+\gamma\langle\theta_t\rangle\bigr).
 \label{eq:fp-adamw}
\end{equation}

\paragraph{Numerical computation.}
For direct hypergradient computation, we use forward mode because it propagates $\partial_\eta\theta_t$ alongside the weights without storing the training trajectory, so memory use does not grow with the number of training steps~\citep{franceschi2017forward}.
We compute inverse-curvature--vector products with conjugate gradients where we can compute without forming the curvature matrix. Unless otherwise stated, we use Gauss--Newton curvature.

\subsection{EoS-LR gradient}
\label{app:eos-lr-implementation}

We evaluate the EoS-LR gradient in Eq.~(\ref{eq_dL_naive}) by approximating the directional derivative of $F$ using a centered finite difference.
At the evaluation point $\theta$, let $u=\nabla\mathcal{L}(\theta)/\|\nabla\mathcal{L}(\theta)\|$ and choose a finite-difference step $h>0$.

For $F_{\mathrm{naive}}(\theta)=2/(\lambda_{\max}(H(\theta))+\gamma)$, we approximate the LR gradient as
\begin{equation}
{\frac{d\mathcal{L}}{d\eta}}=\frac{\|\nabla\mathcal{L}(\theta)\|}{D_h},
\end{equation}
with
\begin{equation}
 D_h:=\frac{F_{\mathrm{naive}}(\theta+hu)-F_{\mathrm{naive}}(\theta-hu)}{2h}.
\end{equation}
This evaluates the intermediate expression in Eq.~(\ref{eq_dL_naive}) without computing third derivatives. We also use this centered-difference approximation to compute the LR gradient based on $F_{\mathrm{batch}}$.

\paragraph{Momentum correction.}
Momentum changes the stability threshold that defines the EoS and therefore requires a corresponding correction to $F$.
In our CIFAR-10 experiments using AirBench, we used Nesterov SGD with momentum coefficient $\beta$. For this optimizer, we use
\begin{equation}
 F_{\mathrm{naive}}^{\mathrm{mom}}(\theta)=
 \frac{2(1+\beta)}{(1+2\beta)\bigl(\lambda_{\max}(H(\theta))+\gamma\bigr)}.
\end{equation}
Replacing $F_{\mathrm{naive}}$ with this expression scales the EoS-LR gradient by $(1+2\beta)/(1+\beta)$.

\subsection{Layerwise FP-LR Gradient with K-FAC}
\label{app:layer-wise-kfac}

For the layer-wise profiles in Figure~\ref{fig:gpt-qknorm-lrgrad-profile}, we use the empirical K-FAC matrix $B_l\otimes A_l$~\citep{martens2015optimizing} as a curvature approximation to the Hessian for layer $l$. Let $a_l\in\mathbb{R}^{d_{\mathrm{in},l}\times n_l}$ and $z_l\in\mathbb{R}^{d_{\mathrm{out},l}\times n_l}$ denote the layer inputs and pre-activation outputs, where $n_l$ is the minibatch size. Write $\delta_l=\partial\mathcal{L}_{\mathrm{tr}}/\partial z_l\in\mathbb{R}^{d_{\mathrm{out},l}\times n_l}$ for the backpropagated loss derivatives. We define $A_l=\langle a_la_l^{\top}/n_l\rangle$ and $B_l=\langle\delta_l\delta_l^{\top}/n_l\rangle$. We average the two factors, the training-loss gradients, and the weights separately over the final training epoch.

For weight-decay coefficient $\gamma$, we compute the FP-LR contribution of layer $l$ as
\begin{equation}
 c_l=-\frac{1}{\eta}(\nabla_{\theta_l}\mathcal{L}_{\mathrm{val}})^{\top}
 \bigl[B_l\otimes A_l+\gamma I\bigr]^{-1}
 \bigl(\langle\nabla_{\theta_l}\mathcal{L}_{\mathrm{tr}}\rangle+\gamma\bar\theta_l\bigr).
\end{equation}
Here, $\nabla_{\theta_l}\mathcal{L}_{\mathrm{val}}$ is the validation-loss gradient for layer $l$ and $\eta$ is the LR. We evaluate linear layers, excluding normalization parameters. We evaluate the inverse-matrix product using eigendecompositions of $A_l$ and $B_l$, without forming their Kronecker product.

\section{Additional Experiments}
\label{app:additional-experiments}

\subsection{EoS-LR Estimates and Gradient Alignment}
\label{sec_C1}

\paragraph{Cross-entropy loss.}
Figure~\ref{fig:cnn-ce-lrgrads} compares EoS-LR and time-averaged FP-LR with finite differences for CNNs trained on CIFAR-10 with cross-entropy loss. Both estimates agree with finite differences over much of the LR range where the loss decreases.

\begin{figure}[!htbp]
    \centering
    \includegraphics[width=1\linewidth]{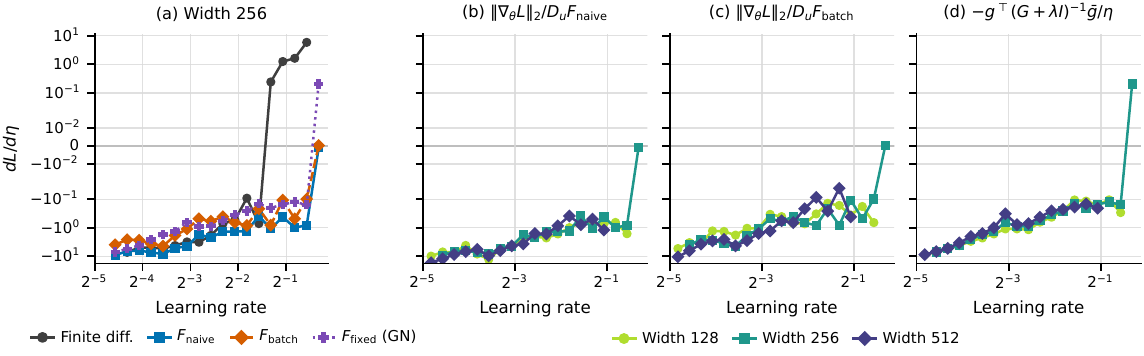}
    \caption{\textbf{LR-gradient estimates for cross-entropy training.} CNNs on CIFAR-10 with widths $128$, $256$, and $512$. (a) At width $256$, we compare finite differences, EoS-LR based on $F_{\mathrm{naive}}$ and $F_{\mathrm{batch}}$, and time-averaged FP-LR using Gauss--Newton curvature. (b--d) The three estimators across widths. The plots show both the sign and magnitude of the LR gradient.}
    \label{fig:cnn-ce-lrgrads}
\end{figure}

\paragraph{ResNet on CIFAR-10.}
EoS-LR also broadly captures the magnitude and LR dependence of the LR gradient for ResNet on CIFAR-10 (Figure~\ref{fig:resnet-eos-lrgrad}).
The remaining discrepancies near zero at larger LRs vary across seeds, suggesting seed-dependent fluctuations rather than a systematic error in the estimated LR dependence.

The largest Hessian eigenvalue lies above the EoS threshold $2/\eta$ at smaller LRs and below it at larger LRs. This observation motivates relaxing the equality constraint in Eq.~(\ref{eq_ori_opt}) to an inequality. Consider first $F(\theta)\leq\eta$, equivalently $\lambda_{\max}\geq2/\eta$. For the Lagrangian $\mathcal{L}(\theta)+\lambda(F(\theta)-\eta)$, the KKT conditions require $\lambda\geq0$, and the local sensitivity relation gives $d\mathcal{L}_*/d\eta=-\lambda\leq0$, where $\mathcal{L}_*(\eta)$ denotes the optimal value, assumed differentiable. Reversing the constraint to $F(\theta)\geq\eta$, equivalently $\lambda_{\max}\leq2/\eta$, yields $\lambda\leq0$ under the same Lagrangian convention and hence $d\mathcal{L}_*/d\eta\geq0$. In either case, complementary slackness implies $\lambda=0$ when the constraint is inactive. A strictly positive or negative LR gradient therefore requires a nonzero multiplier. The two inequality formulations are thus consistent with the decreasing and increasing branches of a U-shaped LR transfer curve, respectively. This interpretation suggests a possible link between the direction of the sharpness constraint relative to the EoS threshold and the shape of the LR transfer curve.

\begin{figure}[!htbp]
    \centering
    \includegraphics[width=\linewidth]{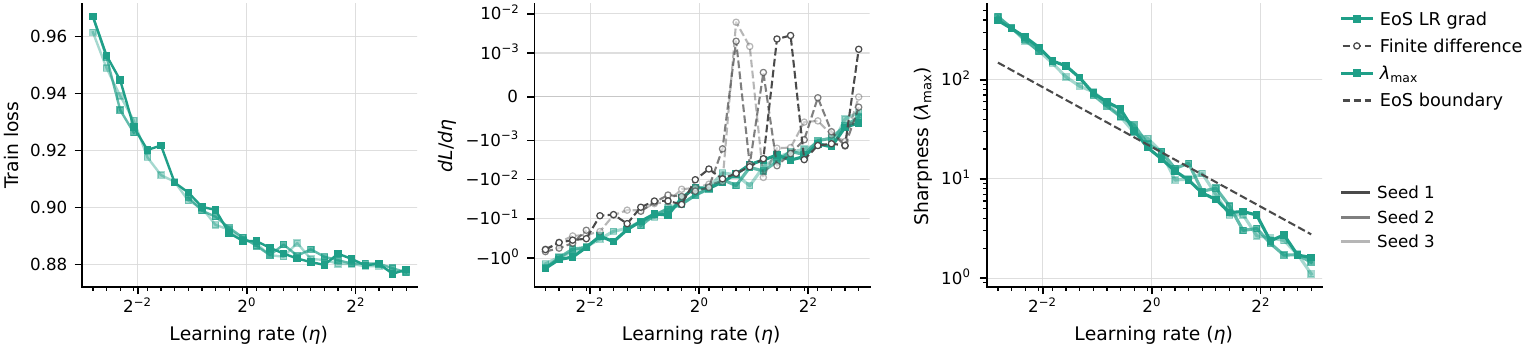}
    \caption{\textbf{EoS-LR estimates for ResNet on CIFAR-10.} Left: training loss versus LR. Middle: EoS-LR estimates and finite-difference LR gradients. Right: the largest Hessian eigenvalue and the EoS boundary $2(1+\beta)/\eta(1+2\beta)$. Each curve corresponds to one of three seeds.}
    \label{fig:resnet-eos-lrgrad}
\end{figure}

\paragraph{Alignment at averaged parameters.}
At time-averaged parameters, the cosine similarity between $\nabla\mathcal L$ and $\nabla\lambda_{\max}$ stabilizes at a positive value (Figure~\ref{fig:alignment-resnet}). As discussed in the main text, this gives a negative EoS-LR gradient in Eq.~(\ref{eq_dL_naive}), consistent with loss decreasing as the LR increases.
When evaluated at the parameters from each training step without averaging, the cosine similarity fluctuates more and can become negative. 
These negative values show that the sign can differ between individual training steps and averaged parameters. This supports the distinction made in the main text between alignment along the optimization trajectory and on the flat manifold.

\begin{figure}[!htbp]
    \centering
    \includegraphics[width=0.8\linewidth]{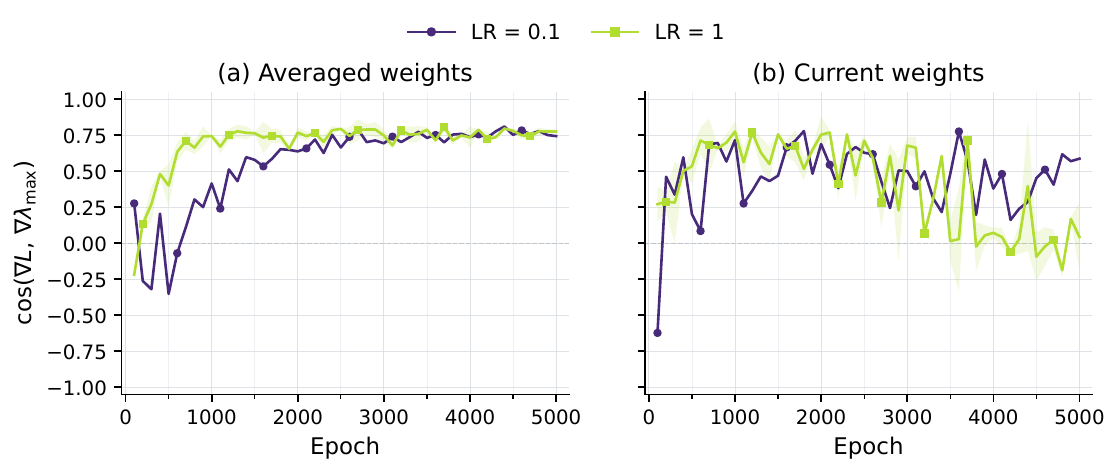}
    \caption{\textbf{Parameter averaging reduces fluctuations in sharpness-gradient alignment.} ResNet on CIFAR-10 at width $64$, trained for $5{,}000$ epochs with LRs $0.1$ and $1$. We plot $\cos(\nabla\mathcal L,\nabla\lambda_{\max})$ at (a) averaged parameters and (b) current parameters. Curves and shaded regions show seed means and standard errors.}
    \label{fig:alignment-resnet}
\end{figure}

\paragraph{LR exponent of the sharpness gradient.}
The sharpness-gradient norm $\|\nabla_\theta\lambda_{\max}\|$ scales with an LR exponent close to $-1$ at all three widths in Figure~\ref{fig:lambda_grad_lr_exponents}. This supports the empirical scaling used in the EoS-LR bound in Section~\ref{sec5_bounds}.

\begin{figure}[!htbp]
    \centering
    \includegraphics[width=0.9\linewidth]{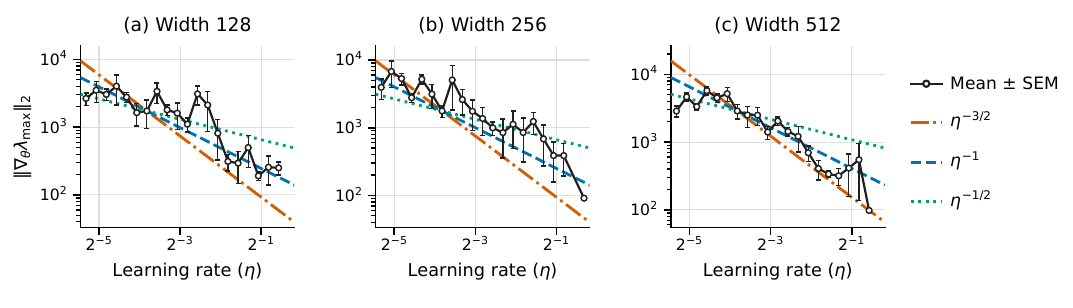}
    \caption{\textbf{LR scaling of the sharpness-gradient norm.} CNNs on CIFAR-10 at widths $128$, $256$, and $512$. We estimate $\|\nabla_\theta\lambda_{\max}\|$ from finite differences of $F_{\mathrm{naive}}$. Reference lines show exponents $-3/2$, $-1$, and $-1/2$; the observed scaling is close to $-1$. Points and error bars show seed means and standard errors.}
    \label{fig:lambda_grad_lr_exponents}
\end{figure}

\paragraph{Batch sharpness at averaged parameters.}
At exponentially averaged parameters, batch sharpness remains below $2/\eta$, whereas the largest Hessian eigenvalue stays close to $2/\eta$ (Figure~\ref{fig:eta-vs-sharpness}). At the final, unaveraged parameters, batch sharpness is closer to $2/\eta$, particularly after longer training. Thus, the two sharpness measures behave differently under parameter averaging in this experiment. We therefore used the final parameters without averaging to compute the EoS-LR gradient with $F_{batch}$.

\begin{figure}[htbp]
    \centering
    \includegraphics[width=\linewidth]{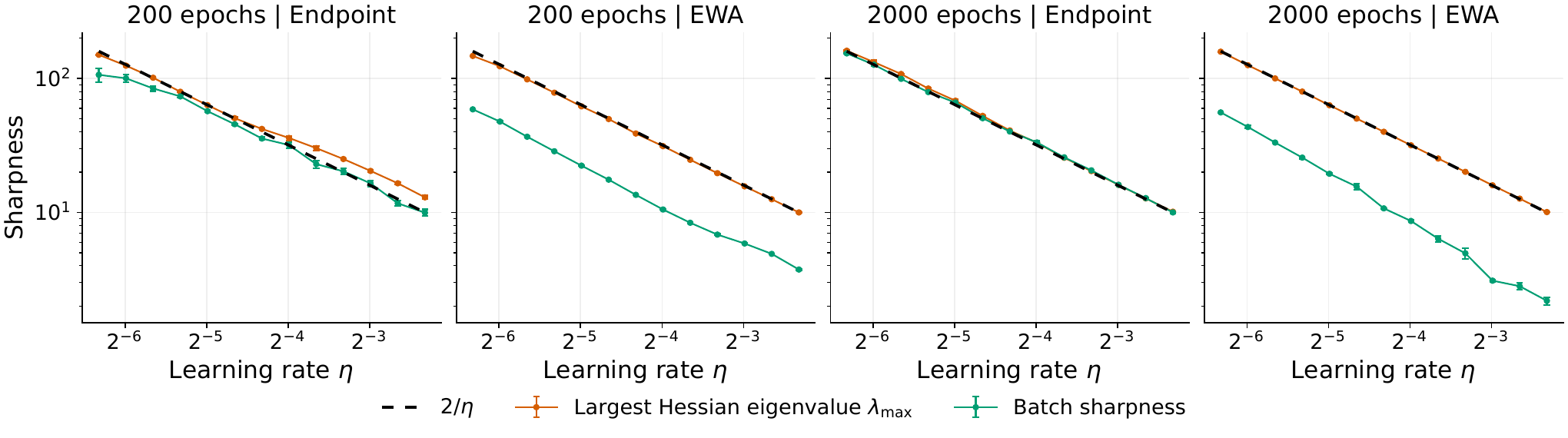}
    \caption{\textbf{Batch sharpness falls below the EoS threshold at averaged parameters.} Sharpness versus LR on Fashion-MNIST using an MLP. The first two panels show results after $200$ epochs, and the last two after $2000$ epochs. Within each pair, the left panel uses the final parameters and the right panel uses exponentially averaged parameters (EWA). We compare the largest Hessian eigenvalue and batch sharpness with $2/\eta$ (dashed).}
    \label{fig:eta-vs-sharpness}
\end{figure}

\paragraph{Longer training under full-batch GD.}
In a full-batch setting, longer training improves agreement between EoS-LR and finite-difference gradients in the small MLP and CNN experiments (Figure~\ref{fig:small-networks-iteration-agreement}).
Both models use full-batch GD on $512$ examples from two classes of Fashion-MNIST, with parameters averaged over the last $512$ updates.
For the MLP, agreement extends toward lower LRs as training proceeds; for the CNN, the large differences at $8192$ iterations are reduced by $32768$ iterations.

This dependence on training iterations is consistent with EoS-LR assuming a constrained minimum in Eq.~(\ref{eq_ori_opt}), where $\nabla\mathcal{L}+\lambda\nabla F=0$.
Figure~\ref{fig:small-networks-iteration-cosine} shows that $\cos(\nabla\mathcal{L},\nabla\lambda_{\max})\approx-1$ at larger LRs, as expected from this condition when the EoS-LR gradient is positive.
This range broadens with training for the MLP; the CNN already exhibits this alignment at many LRs after $8192$ iterations, despite remaining differences between the LR-gradient estimates.

\begin{figure}[htbp]
    \centering
    \includegraphics[width=\linewidth]{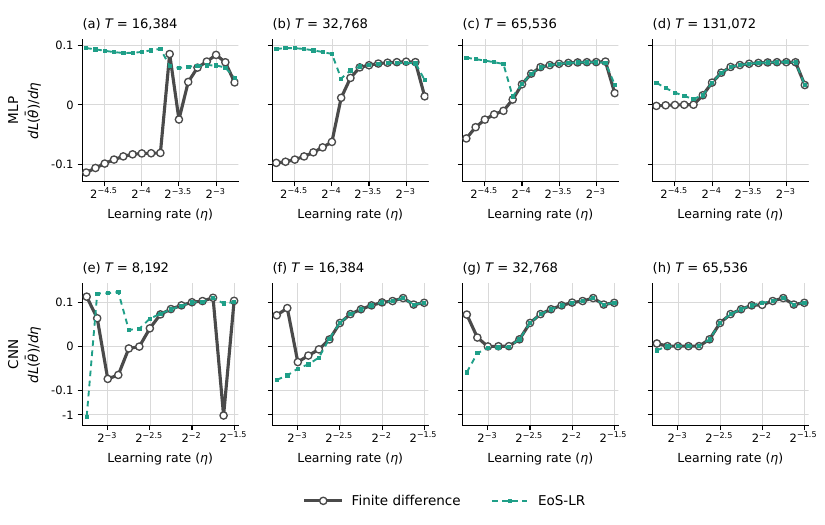}
    \caption{\textbf{LR-gradient agreement with longer full-batch GD training.} Finite differences (gray) and EoS-LR (teal) for the MLP (top) and CNN (bottom) trained for increasing numbers of iterations $T$ on $512$ examples from two classes of Fashion-MNIST. The averaging window remains $512$ updates.}
    \label{fig:small-networks-iteration-agreement}
\end{figure}

\begin{figure}[htbp]
    \centering
    \includegraphics[width=\linewidth]{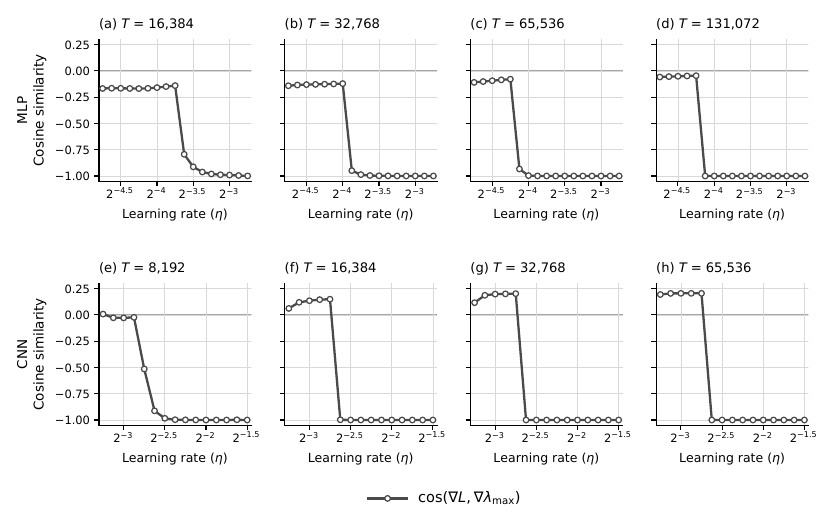}
    \caption{\textbf{Gradient alignment with longer training.} Cosine similarity $\cos(\nabla\mathcal{L},\nabla\lambda_{\max})$ at parameters averaged over the last $512$ updates, for the same MLP (top) and CNN (bottom) as in Figure~\ref{fig:small-networks-iteration-agreement}.}
    \label{fig:small-networks-iteration-cosine}
\end{figure}

\subsection{FP-LR Estimates across Batch Sizes and Architectures}
\label{app-sec:lr-grad-arch}

\paragraph{Alignment of averaged gradients.}
Gradients at epoch-averaged parameters remain closely aligned with data augmentation and batch size $2{,}048$ (Figure~\ref{fig:alignment}). Without augmentation and at batch size $128$, their cosine similarity falls toward zero as the LR increases. These results support slowly changing gradient directions in the former setting. They do not test changes in gradient magnitude, the independence assumption in Appendix~\ref{app:fp-lr-derivation}, or the assumption $\langle\partial_\eta\theta_{t+1}\rangle=\langle\partial_\eta\theta_t\rangle$.

\begin{figure}[!htbp]
    \centering
    \begin{minipage}[c]{0.44\linewidth}
        \includegraphics[width=\linewidth]{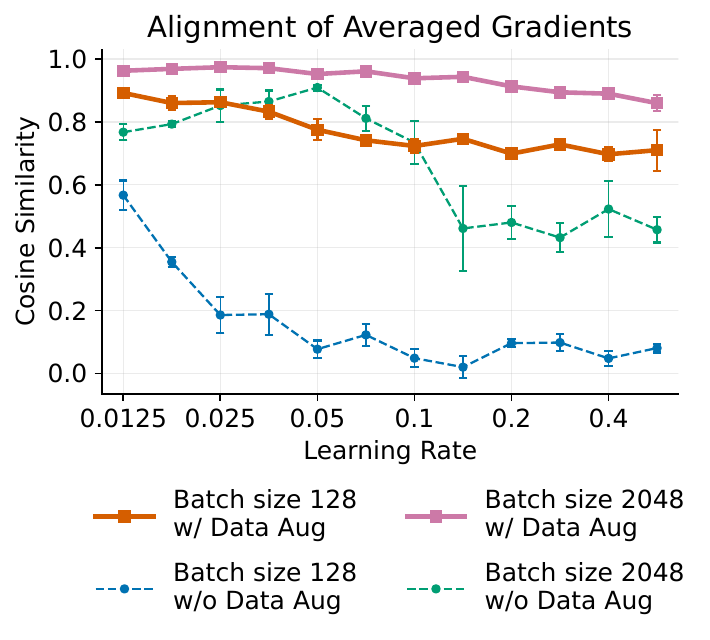}
    \end{minipage}\hfill
    \begin{minipage}[c]{0.5\linewidth}
        \caption{\textbf{Alignment of gradients at successive averaged parameters.} A width-$256$ CNN trained on CIFAR-10 with mean-squared loss for $200$ epochs. We compare batch sizes $128$ and $2{,}048$, with and without augmentation. We plot the cosine similarity of full-data gradients evaluated at successive epoch-averaged parameters near the end of training. Error bars show standard errors across seeds.}
        \label{fig:alignment}
    \end{minipage}
\end{figure}

\paragraph{CNN batch size.}
Time-averaged FP-LR captures the overall decrease in LR-gradient magnitude as the LR increases, across the tested batch sizes (Figure~\ref{fig:app-cnn-bs}). ``Naive Fixed Point'' denotes the fixed-point estimate in Section~\ref{sec4} computed without time averaging. We use either the Hessian or the Gauss--Newton matrix, as indicated in each plot. Using the exact Hessian does not consistently improve agreement with finite differences over the Gauss--Newton approximation. All models are trained for the same number of epochs, so smaller batches result in more parameter updates.

\begin{figure}[!htbp]
    \centering
    \includegraphics[width=1\linewidth]{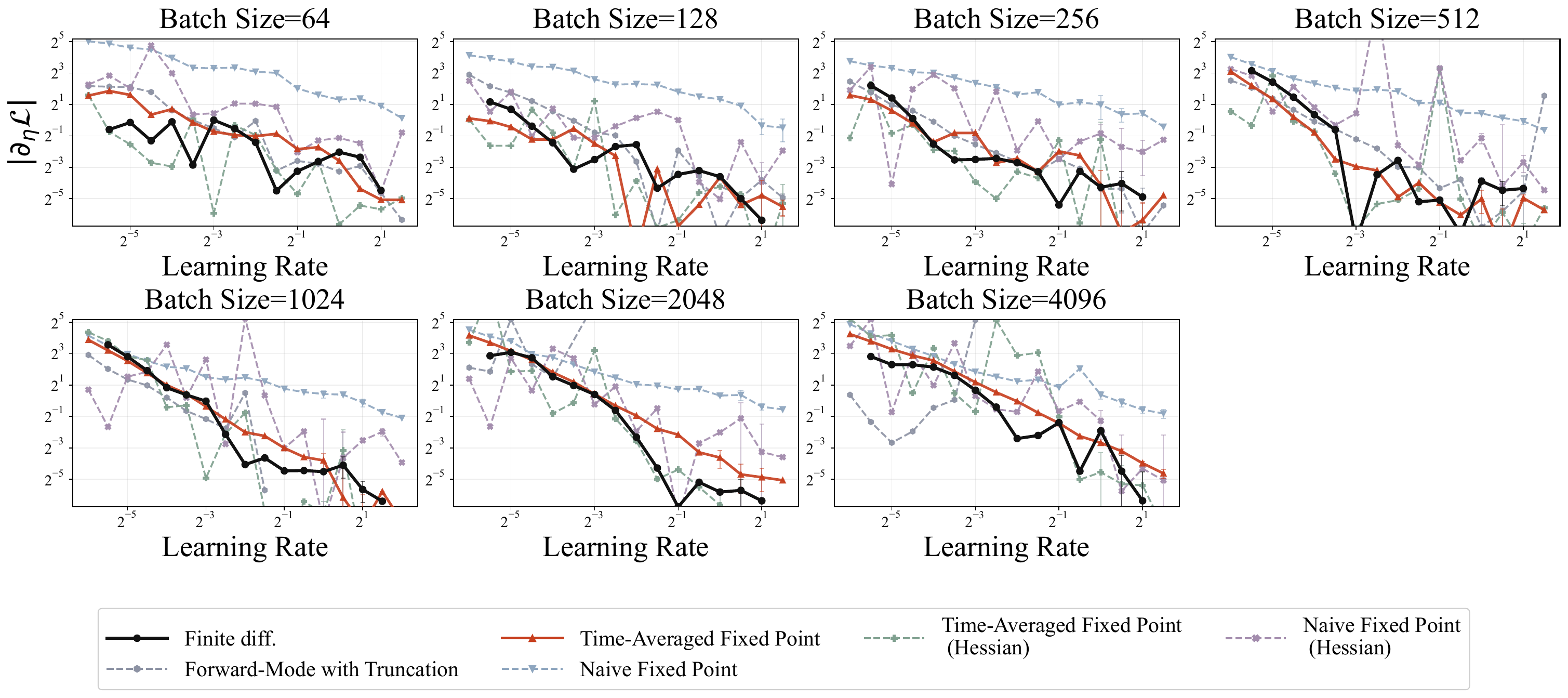}
    \caption{\textbf{FP-LR estimates across batch sizes.} A width-$32$ CNN is trained on CIFAR-10 for $200$ epochs, with batch sizes from $64$ to $4{,}096$. Each panel compares the magnitude of the test-loss LR gradient from finite differences with truncated forward-mode, naive fixed-point, and time-averaged FP-LR estimates. ``Hessian'' distinguishes exact-Hessian estimates from those using Gauss--Newton curvature.}
    \label{fig:app-cnn-bs}
\end{figure}

\paragraph{Fashion-MNIST and ResNet8.}
On Fashion-MNIST, time-averaged FP-LR agrees more closely with finite differences than the naive fixed-point estimate (Figure~\ref{fig:app-gd-fashion}). These plots compare gradient magnitudes and do not establish agreement in sign. For ResNet8, time-averaged FP-LR with the exact Hessian also improves on the naive fixed-point estimate over much of the LR range (Figure~\ref{fig:app-resnet8-lr-macro}). Removing batch normalization increases LR sensitivity.

\begin{figure}[!htbp]
    \centering
    \includegraphics[width=1\linewidth]{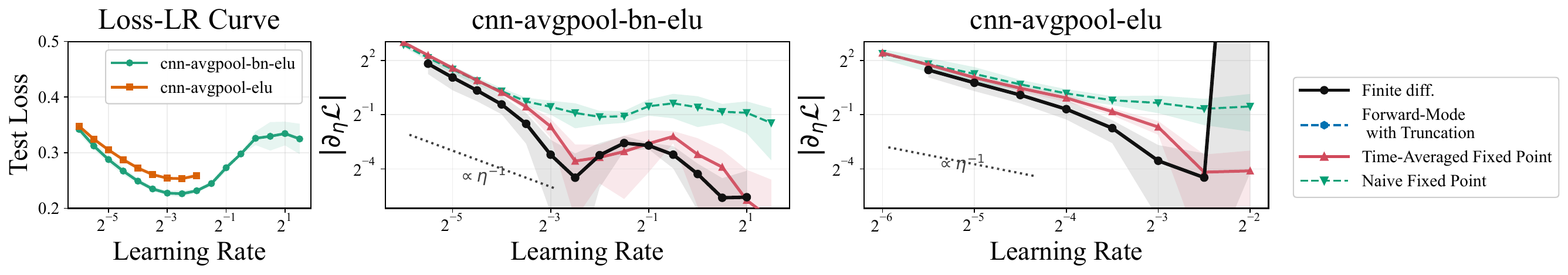}
    \caption{\textbf{LR-gradient estimates on Fashion-MNIST.} Width-$32$ CNNs with and without batch normalization are trained with batch size $1{,}024$. Left: test loss versus LR. Middle and right: LR-gradient magnitudes for the models with and without batch normalization, respectively.}
    \label{fig:app-gd-fashion}
\end{figure}

\begin{figure}[!htbp]
    \centering
    \includegraphics[width=1\linewidth]{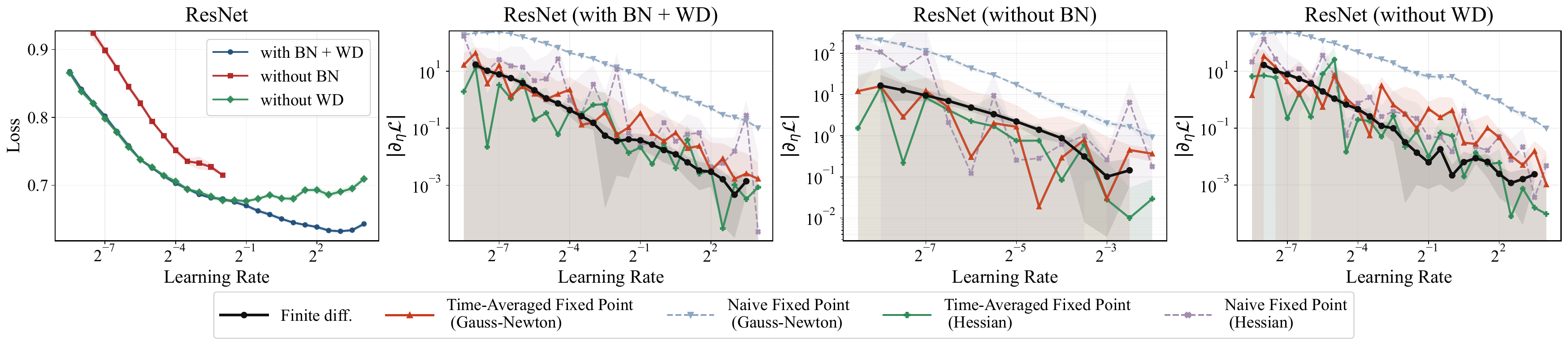}
    \caption{\textbf{Batch normalization reduces ResNet8 LR sensitivity.} ResNet8 is trained on CIFAR-10 for $100$ epochs with SGD, width $8$, and batch size $256$. Left: test-loss curves with BN and weight decay, without BN, and without weight decay. Remaining panels: corresponding LR-gradient magnitudes, using the exact Hessian or the Gauss--Newton approximation.}
    \label{fig:app-resnet8-lr-macro}
\end{figure}

\paragraph{Fixed update budget.}
FP-LR accuracy still varies with batch size when the number of parameter updates is fixed (Figure~\ref{fig:batch_size_mlp_fixed_iters}). For this MLP, the time-averaged Gauss--Newton estimate closely follows finite differences at batch size $4{,}096$, but often overestimates their magnitude at batch size $1{,}024$ and below. The naive Gauss--Newton estimate overestimates sensitivity more strongly. Time averaging therefore improves agreement but does not remove the discrepancies at smaller batch sizes.

\begin{figure}[!htbp]
    \centering
    \includegraphics[width=1\linewidth]{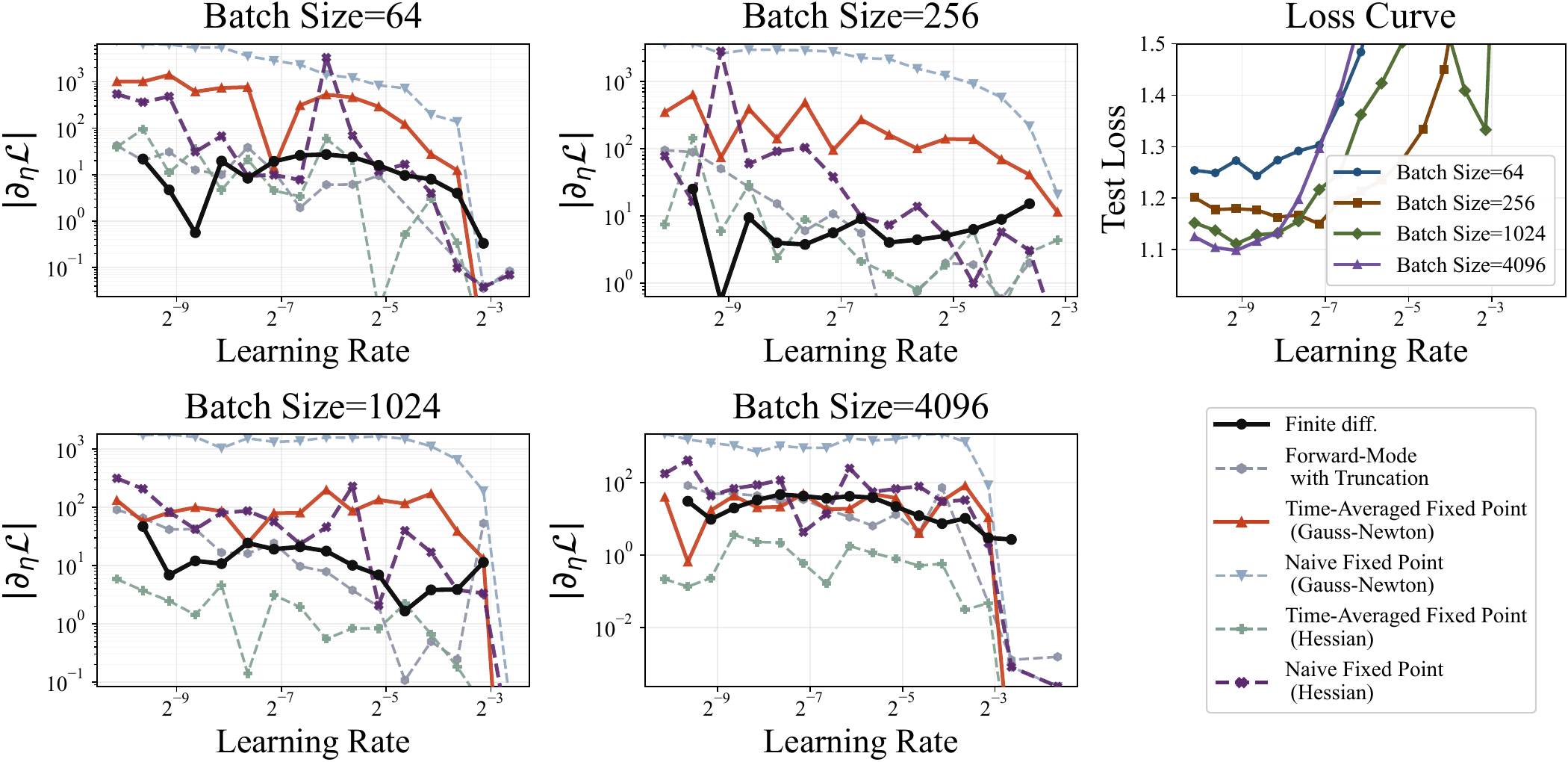}
    \caption{\textbf{MLP LR-gradient estimates with a fixed update budget.} A width-$1{,}024$ MLP with ELU activations is trained on CIFAR-10 for $20{,}000$ updates. We compare batch sizes $64$, $256$, $1{,}024$, and $4{,}096$. Four panels show the magnitude of the test-loss LR gradient; the upper-right panel shows test loss.}
    \label{fig:batch_size_mlp_fixed_iters}
\end{figure}

\subsection{FP-LR Estimates for signGD and AdamW}
\label{app:fp-other-optimizers}

\paragraph{signGD.}
For ResNet8 trained with signGD, time averaging reduces the overestimation of LR-gradient amplitudes by the naive fixed-point estimate (Figure~\ref{fig:app-resnet8-lr-sign}). The time-averaged estimate still differs from finite differences, particularly without batch normalization.

\begin{figure}[!htbp]
    \centering
    \includegraphics[width=1\linewidth]{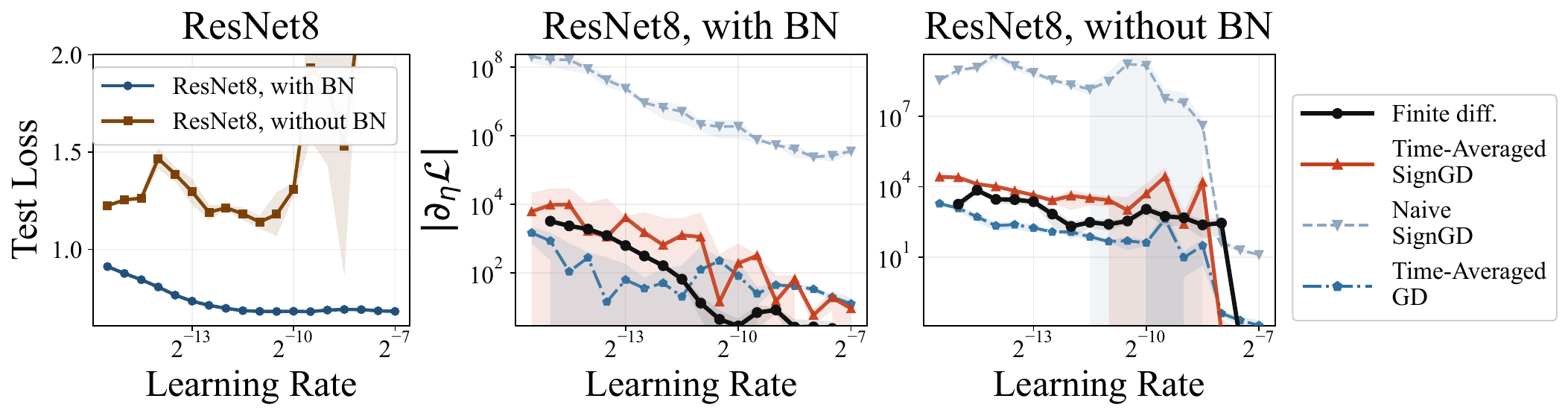}
    \caption{\textbf{Time averaging reduces overestimation by the naive signGD estimator.} ResNet8 on CIFAR-10 with width $8$ and batch size $256$, with and without batch normalization. Left: test loss. Middle and right: LR-gradient magnitudes, comparing finite differences, time-averaged signGD, naive signGD, and the time-averaged GD approximation.}
    \label{fig:app-resnet8-lr-sign}
\end{figure}

\paragraph{AdamW.}
For the CNN trained with AdamW, the time-averaged Adam approximation largely follows the finite-difference (Figure~\ref{fig:app-adamw-cnn}). Agreement is poorer near the minimum, where the loss curve is nearly flat. For ResNet8, using only the output-layer contribution improves agreement over much of the LR range (Figure~\ref{fig:app-adamw-res}). This motivates the layer-wise analysis in Section~\ref{sec5_layer}, although it does not show that contributions from other layers are negligible in general.

\begin{figure}[!htbp]
    \centering
    \includegraphics[width=0.7\linewidth]{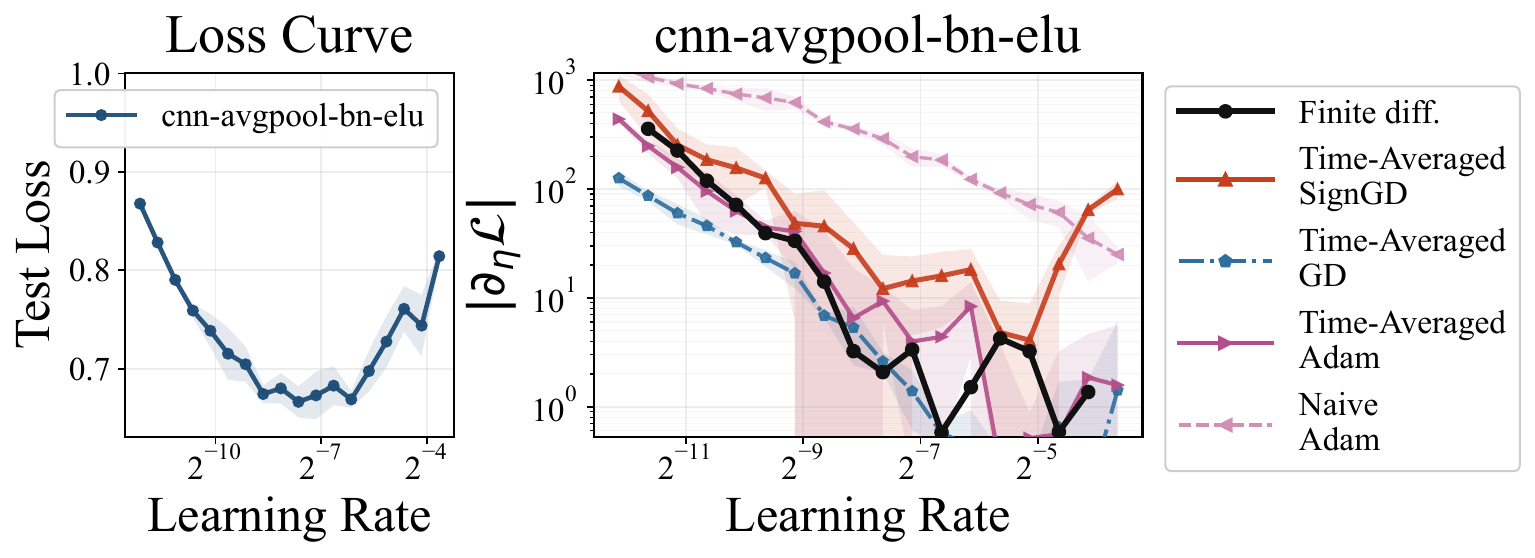}
    \caption{\textbf{Time-averaged estimates for CNN training with AdamW.} A width-$32$ CNN with batch normalization is trained on CIFAR-10 with batch size $1{,}024$. Left: test loss. Right: LR-gradient magnitudes from finite differences and the GD, signGD, and Adam approximations. The naive Adam estimate is included for comparison.}
    \label{fig:app-adamw-cnn}
\end{figure}

\begin{figure}[!htbp]
    \centering
    \includegraphics[width=0.95\linewidth]{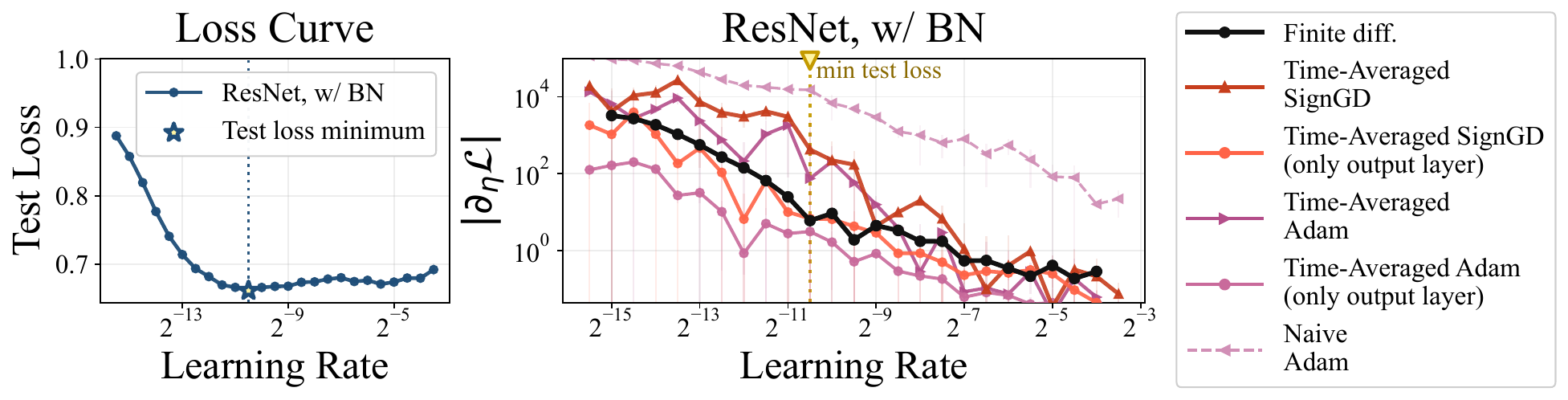}
    \caption{\textbf{Output-layer estimates more closely follow finite differences for ResNet8.} ResNet8 with batch normalization is trained on CIFAR-10 using AdamW, width $8$, and batch size $256$. Left: test loss. Right: LR-gradient magnitudes, including estimates restricted to the output layer. The vertical line marks the LR minimizing test loss.}
    \label{fig:app-adamw-res}
\end{figure}

\subsection{Normalization, Weight Decay, and LR Transfer}

\paragraph{Normalization and LR settings in GPT.}
QK normalization broadens the range of LRs that give low validation loss in GPT models trained with the same LR across layers (Figure~\ref{fig:app-constant-lr-loss}). Without weight decay, layer-wise LRs generally give lower local LR sensitivity than a common LR (Figure~\ref{fig:constant-lr-gpt}). Adding weight decay to the common-LR setting also reduces sensitivity in most configurations. Sensitivity does not consistently increase or decrease with model width.

\begin{figure}[!htbp]
    \centering
    \includegraphics[width=1\linewidth]{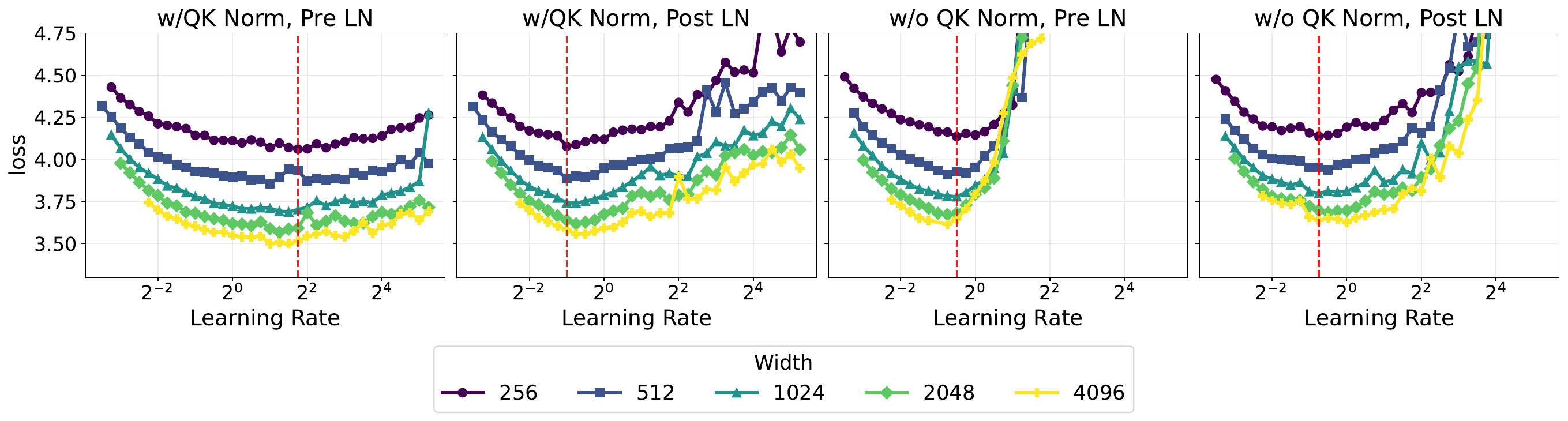}
    \caption{\textbf{QK normalization broadens the low-loss LR range with a common LR across layers.} GPT models trained on FineWeb with AdamW and weight decay. Panels compare Pre-LN and Post-LN, with and without QK normalization; colors indicate model width.}
    \label{fig:app-constant-lr-loss}
\end{figure}

\begin{figure}[!htbp]
    \centering
    \includegraphics[width=1\linewidth]{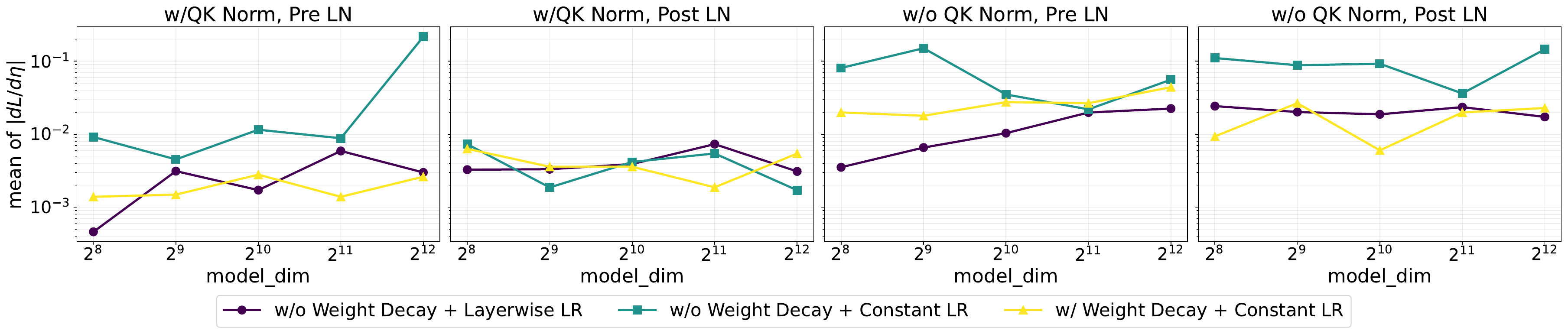}
    \caption{\textbf{Layerwise LRs and weight decay reduce local LR sensitivity in most tested GPT configurations.} Columns compare Pre-LN and Post-LN, with and without QK normalization. The three curves use layer-wise LRs without weight decay, a common LR without weight decay, or a common LR with weight decay. At each model width, sensitivity is the mean of the five smallest nonzero absolute finite-difference estimates of the validation-loss LR gradient across the LR sweep.}
    \label{fig:constant-lr-gpt}
\end{figure}

\paragraph{Batch normalization and weight decay in ResNet.}
Combining batch normalization with LR-independent weight decay gives a broader range of LRs with high test accuracy than using either alone (Figure~\ref{fig:adam-resnet-c10}). At width $32$, the magnitude of the accuracy derivative is also smaller over much of the LR range. Here, sensitivity refers to the derivative of accuracy with respect to LR; the other comparisons use loss derivatives.

\begin{figure}[!htbp]
    \centering
    \includegraphics[width=1\linewidth]{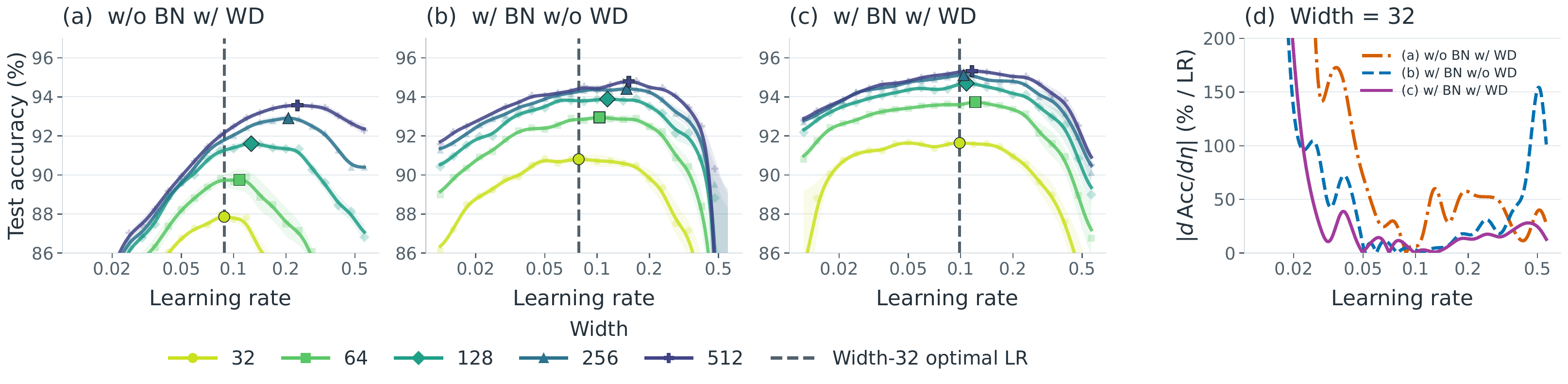}
    \caption{\textbf{Batch normalization and weight decay broaden the high-accuracy LR range.} ResNet on CIFAR-10 trained with AdamW and LR-independent weight decay. (a--c) Test accuracy without BN, without weight decay, and with both, across widths $32$--$512$. Curves are smoothed in log-LR space; faint points show seed means, shaded bands show standard errors, and dashed lines mark the width-$32$ optimum. (d) At width $32$, the absolute finite-difference derivative of smoothed accuracy, measured in percentage points per unit LR.}
    \label{fig:adam-resnet-c10}
\end{figure}

\paragraph{Supplementary results with $1/d$ attention scaling.}
Pre-LN gives lower local LR sensitivity than Post-LN in the GPT experiment with $1/d$ attention scaling, where $d$ is the attention head dimension (Figure~\ref{fig:app-mup-attn}). The difference in local LR sensitivity between models with and without QK normalization is smaller than that between Pre-LN and Post-LN. Validation loss is higher with QK normalization in these runs.

\begin{figure}[!htbp]
    \centering
    \includegraphics[width=1\linewidth]{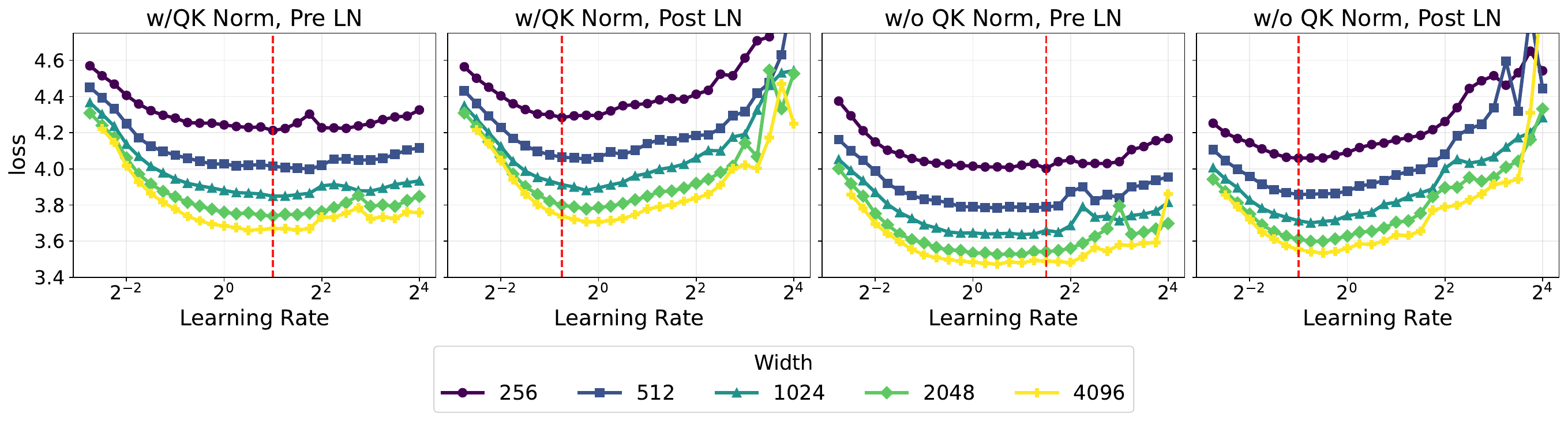}
    \includegraphics[width=1\linewidth]{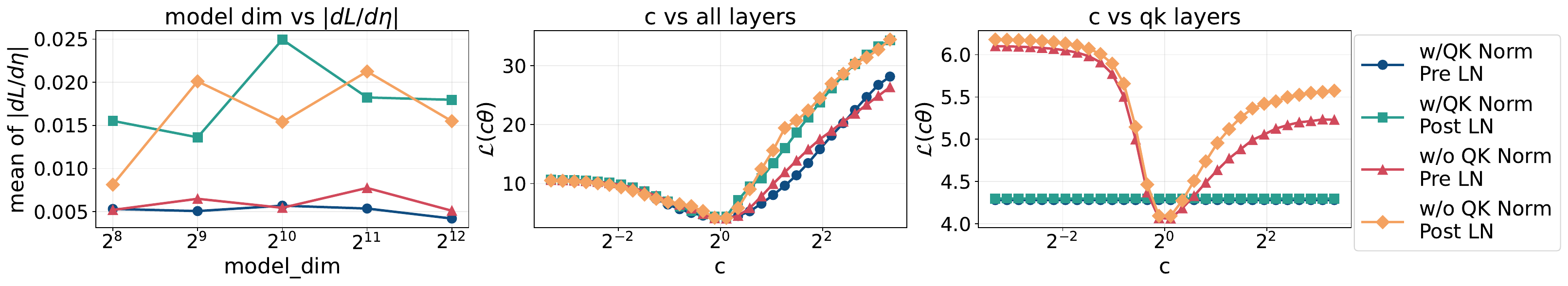}
    \caption{\textbf{Pre-LN reduces local LR sensitivity under $1/d$ attention scaling.} GPT models are trained on FineWeb with AdamW. Top: validation loss versus LR for model widths $256$--$4{,}096$, comparing Pre-LN and Post-LN with and without QK normalization. Bottom left: local LR sensitivity, measured by the mean of the five smallest nonzero LR-gradient magnitudes estimated by finite differences. Bottom middle and right: loss after multiplying all parameters or only QK parameters by $c$, at model width $256$ and LR $0.5$.}
    \label{fig:app-mup-attn}
\end{figure}

The parameter-scaling experiments in the bottom middle and right panels distinguish scale invariance of the QK layers from scale invariance of the whole model. With QK normalization, the loss remains nearly constant when only the QK parameters are rescaled, but changes when all parameters are rescaled. Thus, QK normalization alone does not make the whole model scale-invariant, consistent with the layer-wise discussion in Section~\ref{sec5_norm}. In particular, the QK scaling curves are nearly flat for both Pre-LN and Post-LN, despite their different LR sensitivities.

The two-layer Transformer also has a broader range of LRs with low training loss under Pre-LN than under Post-LN (Figure~\ref{fig:app-2-layer-transformer}).

\begin{figure}[!htbp]
    \centering
    \includegraphics[width=1\linewidth]{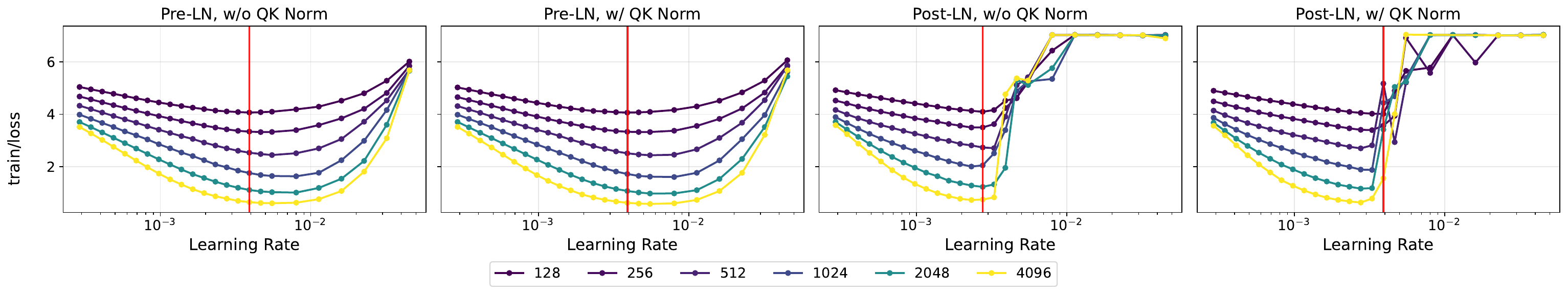}
    \caption{\textbf{Pre-LN gives a broader low-loss LR range in a two-layer Transformer.} We use the implementation of \citet{yang2021tuning} with $1/d$ attention scaling. Panels compare Pre-LN and Post-LN, with and without QK normalization. Curves show training loss for model widths $128$--$4{,}096$.}
    \label{fig:app-2-layer-transformer}
\end{figure}

\section{Experimental Settings}
\label{app-sec:exp-settings}

\subsection{CNN for EoS-LR Experiments}
\label{app:eos-experiment-settings}

The CNN used in Figures~\ref{fig:fig1} and \ref{fig:bs-lrgrad-eos} has three $3\times3$ convolutional layers with $64$, $64$, and $128$ channels, stride $1$, and no padding. Each convolution is followed by ReLU, with $2\times2$ max pooling after the second and third convolutions. The resulting features are flattened and passed through a fully connected layer of width $M$, a ReLU, and a linear output layer with $10$ outputs. The convolutional widths remain fixed when $M$ changes. We do not use batch normalization.

We use a fixed subset of $8{,}192$ CIFAR-10 training images. We do not use data augmentation. We use squared loss. We train with SGD without momentum or weight decay and keep the LR constant for up to $200$ epochs.

For Figure~\ref{fig:fig1}, we use $M\in\{128,256,512\}$, batch size $1024$, and three random seeds. We sweep $21$ LRs from $0.025$ to $0.8$; panel (c) uses $M=256$. For Figure~\ref{fig:bs-lrgrad-eos}, we fix $M=256$ and use batch sizes $16$, $64$, $256$, $1024$, and $4096$, with five random seeds. This sweep uses $18$ LRs from $0.0125$ to approximately $0.635$.

We evaluate training loss and the EoS-LR estimates at the final parameters. 
The finite differences for $F_{\mathrm{naive}}$ and $F_{\mathrm{batch}}$ use step size $h=10^{-3}$, following Appendix~\ref{app:eos-lr-implementation}. For $F_{\mathrm{batch}}$, the sampled batch size matches the training batch size. 

\subsection{CNN and ResNet8 for FP-LR Experiments}
This subsection describes the FP-LR experiments in Figure~\ref{fig:fp-lr-grad-bound} and the additional experiments with CNNs (Figures~\ref{fig:app-cnn-bs}, \ref{fig:app-gd-fashion}, and \ref{fig:app-adamw-cnn}) and ResNet8 (Figures~\ref{fig:app-resnet8-lr-macro}, \ref{fig:app-resnet8-lr-sign}, and \ref{fig:app-adamw-res}). Both models use cross-entropy loss and CIFAR-10, except for the Fashion-MNIST experiment in Figure~\ref{fig:app-gd-fashion}.

The CNN is a two-layer convolutional network in which each convolutional layer has $32$ channels and is followed by an ELU activation and a $2\times 2$ average pooling layer, with a final fully connected layer at the end. We compare two architectural variants: cnn-avgpool-bn-elu, which includes Batch Normalization, and cnn-avgpool-elu, which does not. We use SGD as the optimizer, with a weight decay of $1\mathrm{e}{-5}$ and a mini-batch size of $1024$ for both variants. We train for $200$ epochs, using the standard CIFAR data augmentations of random cropping and horizontal flipping. No learning rate scheduler is used for the CNN experiments.
Training was conducted on a single NVIDIA H100 GPU. One run took approximately 20 minutes to complete 200 epochs, depending on the frequency of metrics collection.

For ResNet8~\citep{he2016deep}, we trained models for 100 epochs with stronger data augmentation, combining random cropping and horizontal flipping with AutoAugment and Cutout. 
The learning rate is decayed according to a cosine annealing schedule.
When decaying the learning rate with a schedule, we calculate the LR gradient using the average learning rate across the entire training run.
The model implementation largely follows the implementations used in \cite{cohen2021gradient}.

\subsection{ResNet Experiments with AirBench}
\label{app:base-width-multipliers}

We use the CIFAR-10 AirBench implementation\footnote{\url{https://github.com/KellerJordan/cifar10-airbench}} with the following modifications and settings. The ResNet has three residual blocks of widths $M$, $3M$, and $4M$, each containing three $3\times3$ convolutions, batch normalization, GELU activations, and max pooling. A fixed $2\times2$ whitening convolution precedes the blocks, and a linear classifier follows them. Whitening filters are estimated from $5{,}000$ training images; their biases are trained for the first three epochs. Batch-normalization scales are fixed at $1$, while biases are trainable.

Training uses $45$ epochs of Nesterov SGD with momentum $0.85$ and cross-entropy loss with label smoothing $0.2$. The LR increases linearly from $20\%$ to its peak over the first $10\%$ of training, then decreases linearly to zero. For Figures~\ref{fig:resnet-eos-lrgrad} and~\ref{fig:alignment-resnet}, we keep the LR constant throughout training. Images are normalized using CIFAR-10 statistics and augmented with horizontal flips, translations of up to four pixels, and $12\times12$ cutout. From each batch of $1{,}024$ images, a proxy network selects the $512$ highest-loss examples for training the main model. The proxy has widths $(32,64,64)$, two convolutions per block, and is updated every fourth batch over the same training iterations.

Batch-normalization biases use $64$ times the optimizer LR of the other parameters. The weight decay is $0.012$, normalized so that changing the LR or the bias LR multiplier does not rescale the weight-decay contribution to the update. Every five optimizer steps, we update an exponential moving average of the weights and resume training from these averaged weights. The averaging decay is $0.95^5(t/T)^3$, where $t$ is the current step and $T$ is the total number of training steps. Figure~\ref{fig:base_width_resnet_cifar} reports accuracy on the CIFAR-10 test split, labeled test accuracy in the implementation, using averaged logits from horizontal flips and one-pixel diagonal translations.

Base Width $M_0$ is the reference width for the $\mu$P scaling. The input and output multipliers, $s_{\mathrm{in}}$ and $s_{\mathrm{out}}$, multiply the first block's initial convolution output and the final logits, respectively. Including the width dependence, the factors are $s_{\mathrm{in}}\sqrt{M/M_0}$ and $s_{\mathrm{out}}\sqrt{M_0/M}$. The initial weights of that convolution are also scaled by $\sqrt{M_0/M}$. These multipliers scale layer outputs, not optimizer LRs directly. Figure~\ref{fig:base_width_resnet_cifar} compares $M_0=1$ and $32$, varying one multiplier at a time while keeping the other at $1$ and holding both fixed across network widths. Specifically, the values $3$ and $1/9$ in panels (b,c) are the width-independent coefficient $s_{\mathrm{out}}$, and the values $3$ and $1/3$ in panels (d,e) are the width-independent coefficient $s_{\mathrm{in}}$. The width-dependent factors $\sqrt{M/M_0}$ and $\sqrt{M_0/M}$ remain as specified above.

\subsection{GPT Training}

For the GPT-based experiments in Figures~\ref{fig:block-lr-resnet-cifar}(c), \ref{fig:lrgrad-exponents}(c), and \ref{fig:gpt-qknorm-lrgrad-profile}(c,d), and the additional GPT experiments, we build on the implementation of \cite{modded_nanogpt_2024} and train on FineWeb using next-token cross-entropy loss. The reference model has $12$ Transformer blocks and $8$ attention heads, with model width $M$ varied across experiments. The MLP expands the hidden dimension to $4M$ and uses squared ReLU activations. Attention uses rotary positional embeddings and a causal sliding-window mask that prevents attention across document boundaries. We retain the value embeddings and learned skip connections of the base implementation; its eighth block omits attention. The normalization layers use RMS normalization. We compare normalization before and after the residual updates (Pre-LN and Post-LN), with and without QK normalization (Figures~\ref{fig:app-constant-lr-loss} and \ref{fig:constant-lr-gpt}).

We use AdamW with $(\beta_1,\beta_2)=(0.8,0.95)$. The LR shown in the figures is a global multiplier of the layer-wise LRs. Under the $\mu$P parametrization, the LRs of the LM head and hidden matrices are multiplied by $M_0/M$, with reference width $M_0=256$. Embedding and scalar LRs remain independent of width. Figures~\ref{fig:app-constant-lr-loss} and \ref{fig:constant-lr-gpt} compare common and layer-wise LR settings, while Figure~\ref{fig:block-lr-resnet-cifar}(c) varies the LR of a selected parameter group. The default attention scaling is $1/\sqrt{d_{\mathrm{head}}}$; the $1/d_{\mathrm{head}}$ variant is used in Figure~\ref{fig:app-mup-attn}.

\enlargethispage{\baselineskip}
Our reference training setup uses sequences of $8192$ tokens. Training lasts $10{,}000$ iterations in total, including $1{,}000$ iterations of linear warmup. The LR then remains constant before a final linear decay. The reference LR-transfer sweeps use $5{,}000$ decay iterations; the LR-gradient experiments also include runs with this phase reduced to one iteration. Validation loss is evaluated on a separate FineWeb validation split. The layer-wise FP-LR profiles in Figure~\ref{fig:gpt-qknorm-lrgrad-profile}(c,d) use the computation described in Appendix~\ref{app:layer-wise-kfac}.
We trained models with one H100 GPU and it took approximately 2-3 hours per training run.

\end{document}